\documentclass[12pt]{article}

\usepackage{amsfonts,amsmath,amsthm,amsbsy,amssymb,dsfont,bm,mathtools,mathalfa}
\usepackage[colorlinks,allcolors=blue]{hyperref}
\usepackage{graphicx}
\usepackage{enumerate}
\usepackage[shortlabels,inline]{enumitem} 
\usepackage{natbib}
\usepackage{multicol} 
\usepackage[boxed,noline]{algorithm2e} 
\SetKwInput{KwData}{Input}
\SetKwInput{KwResult}{Output} 
\usepackage{setspace}

\newcommand{\blind}{1}

\newtheorem{proposition}{Proposition}
\newtheorem{lemma}{Lemma}

\newtheorem{remark}{Remark}

\DeclareMathOperator*{\maxi}{maximize~~}
\DeclareMathOperator*{\mini}{minimize~~}
\DeclareMathOperator{\sign}{sign}

\DeclareMathOperator{\polar}{polar}
\DeclareMathOperator{\ev}{EV}
\DeclareMathOperator{\tr}{tr}

\DeclareMathOperator{\order}{\mathcal{O}}

\newcommand{\1}{\mathds{1}}
\newcommand{\R}{\mathbb{R}}
\newcommand{\U}{\mathcal{U}}
\newcommand{\V}{\mathcal{V}}
\newcommand{\B}{\mathcal{B}}
\newcommand{\cP}{\mathcal{P}}
\newcommand{\T}{^{\text{\normalfont{T}}}}
\newcommand{\F}{_{\text{\normalfont{F}}}}
\newcommand{\norm}[1]{{\left\|#1\right\|}}
\newcommand{\st}{\text{\normalfont{~~subject to~~}}}
\newcommand{\prs}[1]{\texttt{PRS} (#1)}
\newcommand{\sca}[1]{\texttt{SCA}\left(#1\right)}
\newcommand{\sma}[1]{\texttt{SMA}\left(#1\right)}

\begin{document}

\def\spacingset#1{\renewcommand{\baselinestretch}%
{#1}\small\normalsize} \spacingset{1}


\if1\blind
{
  \title{\bf A New Basis for Sparse Principal Component Analysis}
  \author{Fan Chen
    and 
    Karl Rohe\thanks{
    	The authors gratefully acknowledge National Science Foundation grant DMS-1612456 and DMS-1916378 and Army Research Office grant W911NF-15-1-0423.}\hspace{.2cm}\\
    Department of Statistics \\
    University of Wisconsin--Madison}
  \maketitle
} \fi

\if0\blind
{
  \bigskip
  \bigskip
  \bigskip
  \begin{center}
    {\LARGE\bf A New Basis for Sparse Principal Component Analysis}
\end{center}
  \medskip
} \fi

\bigskip
\begin{abstract}

Previous versions of sparse principal component analysis (PCA) have presumed that the eigen-basis (a $p \times k$ matrix) is approximately sparse.  
We propose a method that presumes the $p \times k$ matrix becomes approximately sparse after a $k \times k$ rotation.
The simplest version of the algorithm initializes with the leading $k$ principal components. Then, the principal components are rotated with an $k \times k$ orthogonal rotation to make them approximately sparse. Finally, soft-thresholding is applied to the rotated principal components. 
This approach differs from prior approaches because it uses an orthogonal rotation to approximate a sparse basis. One consequence is that a sparse component need not to be a leading eigenvector, but rather a mixture of them. 
In this way, we propose a new (rotated) basis for sparse PCA.
In addition, our approach avoids ``deflation'' and multiple tuning parameters required for that.
Our sparse PCA framework is versatile; for example, it extends naturally to a two-way analysis of a data matrix for simultaneous dimensionality reduction of rows and columns. 
We provide evidence showing that for the same level of sparsity, the proposed sparse PCA method is more stable and can explain more variance compared to alternative methods. Through three applications---sparse coding of images, analysis of transcriptome sequencing data, and large-scale clustering of social networks, we demonstrate the modern usefulness of sparse PCA in exploring multivariate data.\footnote{An R package, \texttt{epca}, is available online.} 
\end{abstract}

\noindent%
\vfill

\newpage
\spacingset{1.5} 
\section{Introduction} 
Principal component analysis (PCA), introduced in the early 20th century \citep{pearson1901liii, hotelling1933analysis}, is one of the most prevalent tools in exploratory multivariate data analysis. 
PCA projects higher-dimensional data into a lower-dimensional space that is spanned by some uncorrelated principal components (PCs), with the vast majority of the variance in the data kept.
It is, however, commonly conceived that PCs are difficult to interpret  \citep[e.g.,][]{jeffers1967two}, as each PC is a linear combination of many, if not all, original variables.
To remedy such disadvantage, sparse PCA estimates ``sparse'' PCs, each of which consists of a small subset of original variables \citep{zou2018selective}. 

Sparse PCA is originally formulated as an optimization problem over the loading coefficients with a cardinality constraint. Such non-convex constraint results in an NP-hard problem in the strong sense \citep{tillmann2014computational}. In order to circumvent the obstacle, various methods have been proposed, such as the iconic regression-based approach by \citet{zou2006sparse}, 
a convex relaxation to semidefinite programming \citep{daspremont2007direct}, 
the penalized matrix decomposition framework of \citet{witten2009penalized}, 
and the generalized power method due to \citet{journee2010generalized}. 
More recently, theoretical developments of sparse PCA have covered the consistency \citep{johnstone2009consistency, shen2013consistency}, variable selection properties \citep{amini2009high}, rates of convergence, the minimaxity over some Gaussian or sub-Gaussian classes \citep{vu2013minimax, cai2013sparse}, and the statistical-computational trade-offs under the restricted covariance concentration condition \citep{berthet2013optimal, wang2016statistical}.

Despite the extensive literature of sparse PCA, there are two enigmas. 
First, sparse PCA often explains far less variance in the data than PCA does (Figure \ref{fig:1_simu_intro}). 
While this may appear to be a trade-off for sparsity, our results show that a substantial improvement is possible. 
Second, the most common formulations of sparse PCA rely on a matrix deflation after estimating each component. The deflation entails complications of multiple tuning parameters, non-orthogonality, and sub-optimality \citep{mackey2008deflation}. 
Identifiability and consistency present more subtle issues; there is no reason to assume a priori distinct eigenvalues or that the gaps between the eigenvalues are small \citep{vu2013fantope}. Estimating the subspace spanned by multiple sparse PCs at once overcomes this dilemma \citep{vu2013fantope}.

\begin{figure}
	\centering
	\textbf{By allowing for a rotated basis, sparse PCA can explain\\ nearly as much variance as the ordinary PCA.}
	
	\includegraphics[width=.6\linewidth]{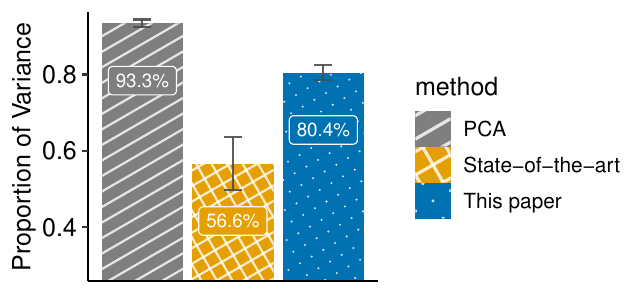}
	\caption{
		Comparison of the proportion of variance explained (PVE) by the 16 PCs estimated by PCA (grey), GPower (yellow, see \citet{journee2010generalized}), and the proposed sparse PCA method (blue).
		For each method, an error bar (based on the three-sigma rule) depicts the variation of PVE over 30 repeats of experiments. 
		More details about the simulated data and settings (e.g., sparsity constraints) are described in Section \ref{sxn:simupve}
	}
	\label{fig:1_simu_intro}
\end{figure}

There are two distinct notions of subspace sparsity: row sparsity and column sparsity \citep{vu2013minimax}. 
Contemporary approaches to sparse PCA primarily focus on row sparsity, which implies that the eigenvectors of the covariance matrix themselves are sparse \citep[e.g.,][]{moghaddam2006generalized}. 
The second notion, column sparsity, is an alternative. 
A column sparse subspace ``\textit{is one which has some orthogonal basis consisting of sparse vectors. This means that the choice of basis is crucial; the existence of a sparse basis is an implicit assumption behind the frequent use of rotation techniques by practitioners to help interpret principal components}'' \citep{vu2013minimax}. 
Row sparsity is the most prevalent notion of sparsity used in contemporary sparse PCA, yet it does not appear to describe many contemporary parametric multivariate models; conversely, many contemporary parametric models in multivariate statistics can be estimated with the sparse PCA approaches that can identify column sparsity \citep{rohe2020vintage}.

In high-dimensional regression, sparse penalties such as the Lasso resolve an invariance; there is an entire space of solutions $b$ which exactly interpolate the data $Y = Xb$ and presuming that the solution $b$ is sparse can make the solution unique. Interestingly, there is no analogue to ``sparsity resolving an invariance'' for the estimation of row sparse subspace, but there is a very clear analogue in estimating column sparse subspace; the basis is determined by the one that provides the most sparse representation of data.

\subsection{Our contributions}\label{sxn:contrib}

\begin{figure}
	\centering
	\textbf{Ordinary PCs without rotations show no sign of sparsity. After an orthogonal rotation, the PCs depict signs of sparsity (with loadings aligning closely with coordinate axes). This shows the prevalence of column sparsity in real data.}
	
	\includegraphics[width=.85\linewidth]{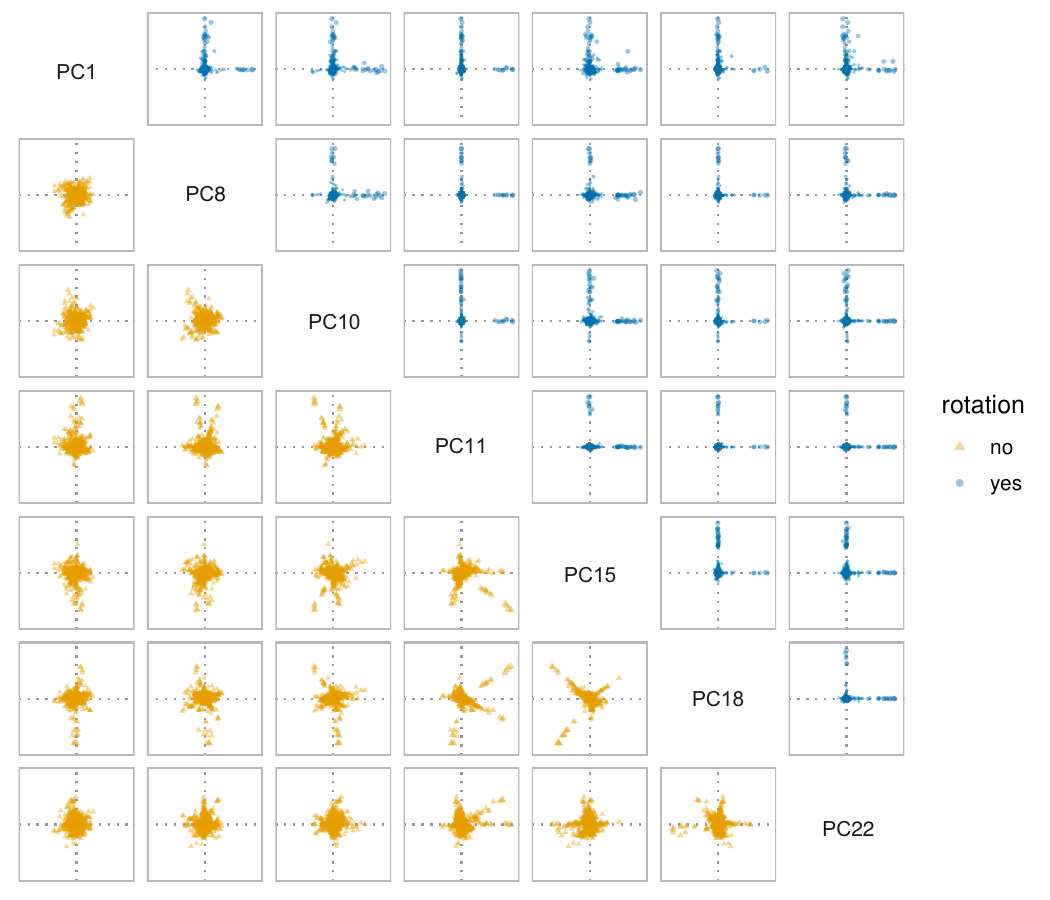}
	\caption{
		Loadings of seven principal components (PCs) from a large scale social network matrix.
		Each (off-diagonal) panel shows the loadings of two PCs on the original variables (displayed as points).
		The lower-triangular panels (yellow) depict the PCs before a rotation.
		The upper-triangular panels (blue) display the PCs after an orthogonal rotation. 
		The PCs before and after the rotation have no special or corresponding relationship.
		In each panel, two perpendicular dotted lines (grey) indicate the coordinate axes. 
		See Section \ref{sxn:twitter} for details about the data analyzed.
	}
	\label{fig:2_rotation}
\end{figure}

In this work, we propose a new method, sparse component analysis (SCA), to estimate multiple PCs that are column sparse. 
The column sparsity is achieved by allowing an orthogonal rotation to PCs prior to imposing any sparsity constraints. 
The algorithm is motivated by two facts.
First, an orthogonal rotation does not affect the total variance explained by a given set of PCs. 
Second, by choosing the orthogonal rotation carefully, PCs can be aligned closely with the coordinate axes, making them approximately sparse (Figure \ref{fig:2_rotation}). 
This technique has been commonly adapted in factor analysis, a close cousin of PCA \citep{thurstone1931multiple, kaiser1960application, jolliffe1995rotation}. 
For example, the varimax rotation \citep{kaiser1958varimax} is a popular choice in the psychology literature.
SCA incorporates the orthogonal rotation and sparsity constraints to find the sparse and orthogonal basis in a subspace (i.e., column sparse PCs). 
We show in Proposition \ref{thm:improve} (Section \ref{sxn:prop}) that 
\begin{quote}
	\em column sparse PCs can explain more variance in the data than row sparse PCs. 
\end{quote}
We validated this with numerical experiments. Additionally, the simulations suggest that SCA is more stable and robust across tuning parameters than existing sparse PCA methods.
Our framework of SCA generalizes naturally to a two-way analysis of a data matrix for simultaneous row and column dimensionality reductions. 
For this, we introduce a low-rank matrix approximation method called sparse matrix approximation (SMA).
The SMA builds on the penalized matrix decomposition previously proposed by \citet{witten2009penalized}.
Furthermore, the SMA provides a unified view of sparse PCA and other modern multivariate data analysis, including sparse independent component analysis \citep[see, e.g.,][]{comon1994independent}. 
Finally, we demonstrate our sparse PCA methods with various high-dimensional data applications, including sparse coding of images, blind source separation, analysis of single-cell transcriptome data, and large-scale clustering of social networks. 
We find compelling evidence for the practical use of our approach, despite concerns about the consistency of PCA in high-dimensions.

\subsection{Organization}
The rest of this paper goes as follows. 
Section \ref{sxn:spca} describes the methods. 
Section \ref{sxn:connection} compares SCA to existing methods.
Section \ref{sxn:simu} compares different sparse PCA methods using simulated data. 
Section \ref{sxn:app} applies SCA to several high-dimensional datasets. 
Section \ref{sxn:disc} concludes the paper with some discussions.

\subsection{Notations}
In this paper, we discuss the \textit{entrywise} matrix norm only. 
For any matrix $A \in\R^{m \times n}$, its entrywise $\ell_p$-norm is defined as
$\norm{A}_{p,p}={(\sum_{i=1}^m \sum_{j=1}^{n} {\left|A_{ij}\right|}^p)}^{1/p}$. 
For simplicity, we also use the notation $\norm{A}_{p}$ for entrywise norm, rather than the norm induced by a vector norm.
In particular, the Frobenius norm (or the Hilbert-Schmidt norm) is then an alias of entrywise $\ell_2$-norm, $\norm{A}\F = \sqrt{\sum_{i=1}^m \sum_{j=1}^n A_{ij}^2}=\norm{A}_2$.
Throughout, the following sets of matrices are frequently considered.
$\U(n)=\{U\in \R^{n\times n} \mid U\T U = U U\T= I_n\}$ denotes all orthogonal (unitary) matrices in $\R^n$. $\V(n,k)=\{V\in \R ^{n\times k}\mid V\T V=I_{k}\}$ represents the Stiefel manifold in $\R^n$, and 
$\B(n,k)=\{V\in\R^{n\times k} \mid V\T V\preceq I_k\}$ 
is its convex hull \citep{gallivan2010note}.

\section{The methods} \label{sxn:spca}

We present a new formulation of sparse PCA as follows. After revisiting PCA, we give the new formulation \eqref{eqn:sca1} (Section \ref{sxn:sca}) and elaborate how it represents column sparsity (Section \ref{sxn:column}) and how it outperforms a row sparsity based method (Section \ref{sxn:prop}). Next, we present an iterative algorithm to compute sparse PCA (Section \ref{sxn:alg}). Lastly, we apply the column sparsity concept to a more general matrix decomposition method (Section \ref{sxn:sma}). 

Consider the data matrix $X\in\R^{n\times p}$ of $n$ observations (or samples) on $p$ variables.
Without loss of generality, we assume that each column of $X$ is centered (i.e. mean-zero) unless otherwise noted.
Throughout this paper, we presume the number of underlying PCs, $k$, is known (see, e.g., \cite{chen2021estimating} for a separated work on estimating $k$ from data using ``cross-validated eigenvalues''). 
PCA finds $k$ uncorrelated linear transformations of the original variables such that after the linear transformations, the most variance is kept. That is, 
\begin{eqnarray} \label{eqn:pca}
	\maxi_Y & \norm{XY}\F  &\st Y\in\V(p,k),
\end{eqnarray}
where the feasible set is the Stiefel manifold, $\V(p,k)$. 
The $j$th PC is the linear combination of original variables whose coefficients are in the $j$th columns of $Y$.
The coefficients are often called \textit{loadings} (or loading coefficients).
Note that loadings are usually non-zero (i.e., $Y$ is usually not sparse).
The transformed data $S=XY\in\R^{n \times k}$ contains the \textit{scores}. That is, $S_{ij}$ is the score of the $i$th sample on the $j$th PC. 

In PCA, PCs are often defined sequentially. That is, in order to find the $k$th PCs, we fix the previous $k-1$ PCs and solve \eqref{eqn:pca}; repeat this for $k=1,2,...$ in order. 
Such definition 
ensures the first $k$ PCs together always explain the most variance in the data.
By contrast, for sparse PCA, we reason in the following that it is sufficient to solve the optimization problem for all PCs at once. 
Note first that the solution to \eqref{eqn:pca} is a subspace, because if $Y^*$ is an optimizer of \eqref{eqn:pca}, then for any orthogonal matrix $R\in\U(k)$, $Y^*R$ is also an optimizer. 
The solution to \eqref{eqn:pca} being a rotation-invariant subspace is desirable because it allows a sparsity-enabling orthogonal rotation to any given solution. 
Importantly, such rotation exists under the assumption of \textit{column sparsity} \citep[see Section \ref{sxn:column} and][]{vu2013minimax}. 
We thereby propose a new method for sparse PCA.

\subsection{Sparse component analysis} \label{sxn:sca}
For sparse PCA, we impose an $\ell_1$-norm constraint\footnote{The $\ell_1$-norm constraint could be replaced by other sparsity constraints, e.g., the $\ell_{0}$-norm analogue.} on the loadings and formulate the following minimization of matrix reconstruction error (MRE)\footnote{MRE depicts the unexplained variation in the data, akin to the sum of squares error in regression.}: 
\begin{eqnarray}\label{eqn:sca1}
	\mini_{Z,B,Y} & & \norm{X -ZBY\T}\F\\
	\st & & Z\in\V(n,k),~ Y\in\V(p,k),~ \norm{Y}_1 \le \gamma,\nonumber
\end{eqnarray}
where $\gamma>0$ is the sparsity controlling parameter, 
and the columns of $Y$ are PC loadings. $ZBY\T$ is an approximation of $X$.

The fundamental difference between formulation \eqref{eqn:sca1} and previous sparse PCA formulations is that the middle $B$ matrix is not necessarily diagonal. Compared to the diagonal $B$ case, this added flexibility has two merits---(i) it allows PCs to be column sparse and (ii) it allows sparse PCs to explain more variance in the data. 

\subsubsection{SCA presumes column sparsity}\label{sxn:column}

Our formulation \eqref{eqn:sca1} presumes the PCs are column sparse. That is, given the subspace of ordinary PCs, there exists an orthogonal rotation, such that after the rotation, the PCs are approximately sparse. 

Let $U D V\T$ be the low-rank singular value decomposition (SVD) of $X$, where $U\in\V(n,k)$ and $V\in\V(p,k)$ contain singular vectors, and $D\in\R^{k \times k}$ is a diagonal matrix with the diagonal entries in decreasing order, and $k\le\min\{n,p\}$ is the rank.
For any two orthogonal matrices $O,R\in\U(k)$, define $Z=UO$, $B=O\T D R$, and $Y=VR$. With these definitions, 
$$X\approx UDV\T = (UO)(O\T DR) (VR)\T = ZBY\T.$$
As such, $ZBY\T$ approximates $X$ as well as $UDV\T$. In particular, the middle $B$ matrix is not diagonal because it absorbs the orthogonal matrices ($O$ and $R$). $Z$ and $Y$ are orthogonally rotated from $U$ and $V$, and both matrices still have orthogonal columns. 
Hence, by imposing an $\ell_1$-norm constraint on $Y$ to make it approximately sparse, we presume that there exists at least one orthogonal basis for the column space of $V$ (i.e., the eigenvectors' subspace), which is not necessarily the original coordinate basis, such that the PCs are sparse under that basis. 

\begin{remark}\label{rmk:ordering}
The formulation of SCA does not explicitly defines an ordering for sparse PCs. 
This is because permuting the columns of $Y$, which can be absorbed by the orthogonal matrix $R$, does not change the approximation of $ZBY\T$. 
As such, the solution to \eqref{eqn:sca1} is not unique.
In practice (see Section \ref{sxn:simupve}), we sort sparse PCs by the explained variance (EV) of individual PCs, which is defined as $\norm{Xy}_2^2$, where $y\in\R^p$ contains the loadings of a PC. As such, the first sparse PC explains the most variation in the data, and the second PC the second most, etc. 
\end{remark}

\subsubsection{Column sparsity versus row sparsity} \label{sxn:prop}

Column sparsity does not assume the loadings of ordinary PCs (i.e., singular vectors of $X$) to be already approximately sparse; they only need to be so after some orthogonal rotations.
By contrast (or more strictly), row sparse PCA presumes that the loadings of ordinary PCs are by themselves approximately sparse (i.e., the singular vectors align closely with the natural coordinate axes already). 

In SCA, the non-diagonal middle $B$ matrix facilitates the more general formulation of column sparse PCA. 
Specially, if $B$ is restricted to diagonal, the formulation reduces to row sparse PCA.\footnote{This restricted formulation is essentially a low-rank SVD with an additional sparsity constraint on the right singular vectors.} 
The next proposition compares column and row sparse PCA in terms of MRE (the proof is simple and provided in Appendix \ref{sxn:proof} for completeness).
\begin{proposition}[Comparison of row and column sparsity] \label{thm:improve}
	Let $X\in\R^{n \times p}$ be any matrix. Suppose $S_Z\subseteq\R^{n \times k}$ and $S_Y\subseteq\R^{p \times k}$ are the feasible sets for $Z$ and $Y$ respectively, where $k \le \min(n, p)$. Then, subject to $Z \in S_Z$, $Y \in S_Y$, and $D$ is diagonal, it holds that
	$$\min_{Z,B,Y} \norm{X-ZBY\T}\F \le \min_{Z,D,Y} \norm{X-Z D Y\T}\F.$$ 
	In particular, the inequality is strict if $S_Z$ and $S_Y$ are defined in \eqref{eqn:sca1}.
\end{proposition} 
Recall that MRE reflects the unexplained variance in the data. Under the same constraints in \eqref{eqn:sca1}, the left-hand side of the inequality corresponds to the MRE objective of column sparse PCA, and the right-hand-side row sparse one. Proposition \ref{thm:improve} says that the solution to column sparse PCA has an optimal MRE strictly less than that of row sparse PCA. In other words, column sparse PCA can capture more variance in the data than row sparse PCA. 

\begin{remark}\label{rmk:B}
	From a parametric perspective, SCA explains more variance because it uses $k^2 - k$ more parameters in the $B$ matrix.  Relative to the total number of parameters, this is typically a small increase; the $Z$ and $Y$ matrices contain roughly $(n+p)k$ parameters, and typically $k$ is much smaller than $n+p$.  Whether these additional parameters in B are statistically justified must be addressed in a case-by-case basis.  In our limited experience with these techniques, the additional parameters are easily justified because the proportion of variance explained dramatically increases (see Section \ref{sxn:simupve}); the output becomes more stable against initializations, perturbations, and tuning parameters (see Section \ref{sxn:simusbm});  and the estimated factors are easily interpretable (see Sections \ref{sxn:scrnaseq} and \ref{sxn:twitter}).
\end{remark}

\subsection{An algorithm for SCA} \label{sxn:alg}
To solve SCA, the following lemma translates \eqref{eqn:sca1} into an equivalent and more convenient form (the proof can be found in Appendix \ref{sxn:proof}).
\begin{lemma}[Bilinear form of SCA]\label{lem:translate1}
	Solving the minimization in \eqref{eqn:sca1} is equivalent to solving the following maximization problem,
	\begin{eqnarray}\label{eqn:sca2}
		\maxi_{Z,Y} & \norm{Z\T X Y}\F & \st Z\in\V(n,k),~ Y\in\V(p,k),~ \norm{Y}_1 \le \gamma.
	\end{eqnarray}
	In particular, for the optimizer in \eqref{eqn:sca1}, $B=Z\T X Y$.
\end{lemma}

Due to the non-convexity of $\ell_2$-equality constraints ($Z\in\V(n,k)$ and $Y\in\V(p,k)$), the feasible set in \eqref{eqn:sca2} is not convex in general. We replace the feasible set with its convex hull using some $\ell_2$-inequality constraints for simplicity, 
\begin{eqnarray} \label{eqn:sca3}
	\maxi_{Z,Y} & \norm{Z\T X Y}\F & \st Z\in\B(n,k),~ Y\in\B(p,k),~ \norm{Y}_1 \le \gamma.
\end{eqnarray}
Due to the Karush-Kuhn-Tucker conditions \citep[see, e.g.,][]{nocedal2006numerical}, one could expect the solution to fall on the boundary (i.e., $Z\in\V(n,k)$, $Y\in\V(p,k)$, and $\norm{Y}_1 = \gamma$) so long as the sparsity parameters are chosen such that $k \le\gamma\le k\sqrt{p}$\footnote{This is for the set $\{Y\in\R^{p \times k}\mid\norm{Y}_1=\gamma\}$ to intersect with the Stiefel manifold $\V(p,k)$.}.

Algorithm \ref{alg:sca} describes an iterative algorithm that computes sparse PCs as formulated in \eqref{eqn:sca3}. 
The input includes a data matrix $X$, the desired number of sparse PCs $k$, and optionally the sparsity controlling parameters $\gamma$. The algorithm outputs the loadings of $k$ sparse PCs. 
In our experiences, a default value of $\gamma=\sqrt{pk}$ appears to generate robust and interpretable sparse PCs (see, e.g., Section \ref{sxn:simusbm}). We discuss a data-driven method of tuning the sparsity parameters in Supplementary Section S1. 
In general, Algorithm \ref{alg:sca} does not necessarily converge to a global optimum for \eqref{eqn:sca3}; however, our empirical studies indicate that the algorithm does converge to interpretable factors for appropriate choices of the sparsity parameters. Note that each iteration results in a decrease in the objective.

The SCA algorithm initializes $Z\in\V(n,k)$ and $Y\in\V(p,k)$ with the top $k$ left and right singular vectors of $X$ respectively.
Once initialized, the algorithm alternatively updates $Z$ and $Y$; fixing one and optimizing the other until convergence.
The iteration is because the objective function is bilinear in $Z$ and $Y$, allowing for fast updates. 
Specifically, with $Y$ fixed, \eqref{eqn:sca3} takes the form 
\begin{eqnarray} \label{eqn:sca4z}
	\maxi_{Z} & \norm{Z\T X Y}\F & \st Z\in\B(n,k).
\end{eqnarray}
With $Z$ fixed, \eqref{eqn:sca3} takes the form 
\begin{eqnarray} \label{eqn:sca4y}
	\maxi_{Y} & \norm{Z\T X Y}\F & \st Y\in\B(p,k),~ \norm{Y}_1\le \gamma.
\end{eqnarray}

\begin{algorithm}
	\DontPrintSemicolon
	\caption{Polar-Rotate-Shrink (PRS)}
	\label{alg:prs}
	\KwData{$A\in\R^{p \times k}$,\\
		\Indp\Indp sparsity parameter $\gamma$ (optional, default to $\sqrt{pk}$)\;}
	\textbf{Procedure} $\prs{A}$\textbf{:}\;
	\Indp
	$\tilde{Y}\leftarrow$ left singular vectors of $A$\;
	$Y^*\leftarrow$ rotate $\tilde{Y}$ with varimax\tcp*{Section \ref{sxn:rotation}}
	$\hat{Y} \leftarrow$ soft-threshold $Y^*$ with parameter $\gamma$\;
	\Indm
	\KwResult{$\hat{Y}$}
\end{algorithm}

\begin{algorithm}
	\DontPrintSemicolon
	\caption{Sparse Component Analysis (SCA)}
	\label{alg:sca}
	\KwData{Data matrix $X$ and a number of components $k$}
	\textbf{Procedure} $\sca{X,k}$\textbf{:}\;
	\Indp
	Initialize $\hat{Z}$ and $\hat{Y}$ with the top $k$ left and right singular vectors of $X$\;
	\Repeat{convergence}
	{
		$\hat{Y} \leftarrow\prs{X\T \hat{Z}}$\tcp*{Algorithm \ref{alg:prs}}
		$\hat{Z} \leftarrow \polar(X \hat{Y})$\tcp*{Lemma \ref{lem:unconstrained}}
	}
	\Indm
	\KwResult{Sparse loadings $\hat{Y}$}
\end{algorithm}

\subsubsection{Update $Z$ fixing $Y$}
The update of $Z$ fixing $Y$ in \eqref{eqn:sca4z} is algebraic. The following lemma provides a set of solutions to \eqref{eqn:sca4z}, which is extended from Theorem 7.3.2 in \citet{horn1985matrix} (the proof is included in Appendix \ref{sxn:proof} for completeness).
\begin{lemma}[Maximization without sparsity constraint] \label{lem:unconstrained}
	Given a full-rank matrix $X\in\R^{n\times p}$, with $p\le n$, let the singular values of $X$ be $\sigma_i$ for $i=1,2,...,p$. Then,
	\begin{eqnarray*}
		\max_{Y\in\V(n, p)} & \norm{ X\T Y }\F ~=  & \sum_{i=1}^p \sigma_i
	\end{eqnarray*}
	with the maximizer $Y^*=\polar(X)$, up to any orthogonal rotation from the right. Here, $\polar(X)=X(X\T X)^{-1/2}$ is the \textit{polar} of $X$.
\end{lemma}
Due to Lemma \ref{lem:unconstrained}, the SCA algorithm updates $Z$ with the polar of $XY$, $\hat{Z}=\polar(XY)$, which can be computed in $\order(nk)$ time \citep{journee2010generalized}. 

\subsubsection{Update $Y$ fixing $Z$}
To update $Y$ fixing $Z$, we start by solving the non-sparse version of \eqref{eqn:sca4y} (i.e., remove the sparsity constraint $\norm{Y}_1\le \gamma$), 
\begin{eqnarray} \label{eqn:sca4y_unconstrained}
	\maxi_Y & \norm{Z\T X Y}\F & \st Y\in\B(p,k).
\end{eqnarray}
Let $\tilde{Y}=\polar(X\T Z)$. Then, $\tilde{Y}$ is one element in the subspace of the solutions to \eqref{eqn:sca4y_unconstrained}.
Before imposing the sparsity constraint, we look for an orthogonal rotation $R$ to $\tilde{Y}$ to minimize ${\|\tilde{Y}R\|}_1$. 
However, $\norm{Y}_1$ is not a smooth function of $Y$ if it contains at least one zero entry, entailing the complications of defining sub-gradients.
Alternatively, the SCA algorithm minimizes a smoother criterion based on the $\ell_{4/3}$ norm:
\begin{eqnarray} \label{eqn:sca4y_varimax}
	\mini_R & \norm{\tilde{Y}R}_{\frac{4}{3}} & \st R\in\U(k).
\end{eqnarray}
This sub-problem leads to the varimax rotation (see Section \ref{sxn:rotation}) that is widely applied in factor analysis \citep{kaiser1958varimax}.
We denote $Y^*=\tilde{Y}R^*$ to be the orthogonally rotated solution to \eqref{eqn:sca4y_unconstrained}, where $R^*$ is the solution to \eqref{eqn:sca4y_varimax}.
Finally, considering the $\ell_1$-norm sparsity constraint, we apply the element-wise soft-thresholding of $Y^*$ with the sparsity parameter $\gamma$, which is defined as \citep{donoho1995noising, tibshirani1996regression} 
\begin{equation} \label{eqn:soft}
	\left[T_\gamma(Y^*)\right]_{ij} = \sign(Y^*_{ij}) \cdot \left(|Y^*_{ij}|-t\right)_+,
\end{equation} 
where $t>0$ is the threshold determined by the equation $\norm{T_\gamma(Y^*)}_1=\gamma$, and $x_+$ equals $x$ if $x > 0$ or 0 otherwise.
We discuss several properties of soft-thresholding in Supplementary Section S2. 
In summary, the update of $Y$ given $Z$ consists of three steps that we call ``Polar-Rotate-Shrink'' (PRS, Algorithm \ref{alg:prs})---first, compute a solution to the unconstrained problem \eqref{eqn:sca4y_unconstrained}; second, rotate with varimax; third, soft-threshold all of the elements.\footnote{More investigation is needed in order to understand the statistical properties of PRS. For example, in a recent paper \citep{rohe2020vintage}, we showed that PCA with the varimax rotation is a consistent estimator for a broad class of modern factor models, that includes the degree corrected stochastic block model \citep{karrer2011stochastic}.}

\subsubsection{Orthogonal rotations: Varimax and quartimax} \label{sxn:rotation}
For any matrix $A\in\R^{p \times k}$, the \textit{varimax criterion} is defined as the sum of column (sample) variance of squared elements ($A_{ij}^2$) \citep{kaiser1958varimax}:
\begin{equation*} 
	C_\text{varimax}(A) = \sum_{j=1}^k\left[\frac{1}{p}\sum_{i=1}^p A_{ij}^4-\frac{1}{p^2}\left(\sum_{i=1}^p A_{ij}^2\right)^2\right].
\end{equation*}
For a fixed matrix $Y\in\R^{p \times k}$, the \textit{varimax rotation} seeks an orthogonal rotation $R\in\R^{k \times k}$ to maximize the varimax criterion evaluated at $YR$, 
\begin{eqnarray}\label{eqn:varimax}
	\maxi_R & C_\text{varimax}(YR) & \st R\in\U(k).
\end{eqnarray}
It is commonly used in factor analysis for producing nearly sparse and interpretable loadings of PCs, especially in the psychology literature. 
The varimax rotation is easy to compute; for example, the base function \texttt{varimax} in \texttt{R} implements a gradient projection algorithm of it \citep{bernaards2005gradient}.
\citet{jennrich2001simple} showed that the gradient projection algorithm converges to a local optimum from any starting point and enjoys geometric (or linear) convergence rate. 

The varimax criterion naturally links to the $\ell_{4/3}$-norm objective function in \eqref{eqn:sca4y_varimax}.
Since $Y\in\V(p,k)$, the columns of $Y$ have unit length. Hence, $\sum_{i=1}^p Y_{ij}^2=1$, and the varimax criterion reduces to a simpler form (also known as the \textit{quartimax} criterion as introduced by \citet{carroll1953analytical}) up to an additive constant:
\begin{equation*} \label{eqn:quartimax}
	C_\text{quartimax}(Y) = \sum_{i=1}^p \sum_{j=1}^k Y_{ij}^4=\norm{Y}_4^4,
\end{equation*}
which is the $\ell_4$-norm of $Y$ to the power of 4.
Next, by the H\"older's inequality (using the H\"older conjugates $4/3$ and $4$) and the power mean inequality (and that $\norm{Y}\F=\sqrt{k}$), 
$\norm{Y}_\frac{4}{3} {\norm{Y}_{4}} \ge \norm{Y}_1 \ge \norm{Y}_F =\sqrt{k}.$
This implies that maximizing the varimax criterion is the dual problem of minimizing the $\ell_{4/3}$-norm objective.
Hence, to update $Y$ in the algorithm of SCA, we invoke the varimax rotation in \eqref{eqn:varimax} as a proxy of \eqref{eqn:sca4y_varimax}.

\begin{remark}\label{rmk:absmin}
Besides varimax, we experimented the orthogonal rotation that directly minimizes the $\ell_1$ norm, which we call the ``\textit{absmin}'' rotation: 
\begin{eqnarray}\label{eqn:absmin}
	\mini_R & \norm{YR}_1 & \st R\in\U(k).
\end{eqnarray}
However, the objective function is not smooth at those $R$ where $YR$ contains at least one zero element; this posts challenges to solving \eqref{eqn:absmin}. 
For example, we tried a gradient projection algorithm using the gradient direction $Y\T\sign(YR)$, where $\sign(\cdot)$ is the element-wise sign function, yet the algorithm hardly converges. 
It is worth noting that in our limited experiments, where we used the absmin rotation but only allowed fifteen iterations of this gradient projection algorithm, we obtained marginally better solutions, in terms of explained variance, than using the varimax rotation (see Section \ref{sxn:simupve}). 
It is of future interest to investigate alternative orthogonal rotations that are easy to compute and can generate approximately sparse structure. 
\end{remark}

\subsection{Sparse matrix approximation} \label{sxn:sma}

In the SCA algorithm above, a sparsity constraint can also be applied to $Z$, in addition to $Y$. We call this sparse matrix approximation (SMA).
We define SMA as the solution to a matrix reconstruction error minimization problem:
\begin{eqnarray}\label{eqn:sma1}
	\mini_{Z,B,Y} & & \norm{X- Z B Y\T}\F \\ 
	\st & & Z\in\B(n,k),~ \cP_1(Z) \le \gamma_z,~ \nonumber \\
	& & Y\in\B(p,k),~ \cP_2(Y) \le \gamma_y,\nonumber
\end{eqnarray}
where $\gamma_z>0$ and $\gamma_y>0$ are the sparsity controlling parameters, and $\cP_1$ and $\cP_2$ are some \textit{penalty} functions that promote sparsity. 
If $\gamma_Z$ is so large that $\cP_1(Z) \le \gamma_z$ is always satisfied, then \eqref{eqn:sma1} is equivalent to SCA.
Similar to Lemma \ref{lem:translate1}, we transform \eqref{eqn:sma1} into an equivalent and more convenient form (the proof is almost identical to that of Lemma \ref{lem:translate1} thus is omitted),
\begin{eqnarray}\label{eqn:sma2}
	\maxi_{Z,Y} & & \norm{Z\T X Y}\F \\ 
	\st & & Z\in\B(n,k),~ \cP_1(Z) \le \gamma_z,~  \nonumber \\
	& & Y\in\B(p,k),~ \cP_2(Y) \le \gamma_y.\nonumber
\end{eqnarray}
The two criteria in \eqref{eqn:sma1} and \eqref{eqn:sma2} are equivalent if and only if $B=Z\T X Y$. 
We interpret $B$ as the ``\textit{score}'' of SMA, since the solution to \eqref{eqn:sma1} maximizes the sum of squares of its elements, $\sum_{i,j}B_{ij}^2$. 
It is also worth noting that the squared matrix reconstruction error equals to $\norm{X}\F^2-\norm{B}\F^2$ (see the proof of Lemma \ref{lem:translate1}).

Since SMA is a simple extension from SCA, we extend Algorithm \ref{alg:sca} for SMA in Algorithm \ref{alg:sma}, where we apply PRS to $Z$ in addition to $Y$.  
The output includes the estimated $Z$, $B$, and $Y$.

\begin{algorithm}
	\DontPrintSemicolon
	\caption{Sparse Matrix Approximation (SMA) with $\cP_1\left(A\right) = \cP_2\left(A\right) = \norm{A}_1$.}
	\label{alg:sma}
	\KwData{data matrix $X\in\R^{n \times p}$ and the approximation rank $k$}
	\textbf{Procedure} $\sma{X, k}$\textbf{:}\;
	\Indp
	Initialize $\hat{Z}$ and $\hat{Y}$ with the top $k$ left and right singular vectors of $X$\;
	\Repeat{convergence}
	{
		$\hat{Z}\leftarrow\prs{X\hat{Y}}$\tcp*{Algorithm \ref{alg:prs}}
		$\hat{Y}\leftarrow\prs{X\T\hat{Z}}$\tcp*{Algorithm \ref{alg:prs}}
	}
	$\hat{B} \leftarrow \hat{Z}\T X \hat{Y}$\;
	\Indm
	\KwResult{$\hat{Z}$, $\hat{B}$, and $\hat{Y}$}
\end{algorithm}

We highlight that SMA generalizes the popular penalized matrix decomposition (PMD) proposed by \citet{witten2009penalized}, which is also similar to the method of \citet{shen2008sparse}.
The PMD also approximates a data matrix $X\in\R^{n \times p}$ by the product of three matrices, $ZDY\T$, where $Z\in\V(n,k)$ and $Y\in\V(p,k)$ are presumed sparse, and $D\in\R^{k \times k}$ is a diagonal matrix whose diagonal entries are in decreasing order, and $k$ is the rank of the matrix approximation.
For sparsity, PMD applies penalty functions to $Z$ and $Y$, leading to the matrix reconstruction error minimization formulation of PMD:\footnote{The paper originally considers the PMD with $k=1$. The PMD finds multiple factors sequentially using a deflation technique.}
\begin{eqnarray*}\label{eqn:pmd1}
	\mini_{U,D,V} & & \norm{X-ZDY\T}\F \\
	\st & & Z\in\B(n,k),~ \cP_1(Z) \le \gamma_z,~ \nonumber \\ 
	& & Y\in\B(p,k),~ \cP_2(Y) \le \gamma_y,~ \nonumber\\
	& & D \text{ is diagonal,}\nonumber
\end{eqnarray*}
where $\gamma_z,\gamma_y>0$ are parameters that control the sparsity of $Z$ and $Y$, and $\cP_1$ and $\cP_2$ are some convex penalty function (e.g. $\ell_1$-norm). 

The single difference between SMA and PMD is the the diagonal constraint on the middle matrix. In this way, SMA generalizes PMD, because, SMA estimates $k^2-k$ more parameters in $B$ than PMD (see Remark \ref{rmk:B}).
Proposition \ref{thm:improve} suggests that the reconstruction error of SMA is less or equal to that of PMD (see also Remark \ref{rem:pmd} in Appendix \ref{sxn:proof}). 
Algorithmically, in order to compute PMD, \citet{witten2009penalized} proposed to find the solution by sequentially maximizing $B_{ii}$ for $i=1,2,...,k$ (recall that $B=Z\T X Y$). By contrast, solving the SMA in \eqref{eqn:sma2} amounts to maximizing the entirety of the score matrix, that is, $\norm{B}\F$.

\section{Connections to existing methods} \label{sxn:connection}

As mentioned in Section \ref{sxn:contrib}, SCA is related to factor analysis in that they both use a rotation. One key difference is that the sparsity constraint in SCA creates actual zeros. In this section, we compare SCA with several existing methods of sparse PCA. Then, we introduce two existing data processing techniques that are related to SCA. 

\subsection{Existing sparse PCA methods} 

The formulation of SCA is akin to multiple existing sparse PCA formulations. However, the possibility of orthogonal rotations has not been explored thoroughly, despite the plethora of available methods. In this section, we elucidate these connections and point to some differences. 

\begin{description}
	
	\item[SPCA \citep{zou2006sparse}] SPCA is motivated to maximize the explained variance in the data \citep{jolliffe2003modified}. The formulation of SPCA minimizes a ``residual sum of squares plus penalties'' type of criterion,
	\begin{eqnarray*} \label{eqn:spca}
		\mini_{U,V} &  \norm{X - XVU\T }^2\F+\lambda_1\norm{V}\F^2+\sum_{j=1}^{k}\lambda_{2,j}\norm{v_j}_1 &\st U\in\V(p,k),
	\end{eqnarray*}
	where $v_j$ is the $j$th column of $V\in\R^{p\times k}$ containing the sparse loadings of the $j$th PC, and $\lambda_1$ and $\lambda_{2,j}$ are tuning parameters. 
	In this formulation, the first and the third terms are not invariant to orthogonal rotations (on $V$). Specially, the first term $\norm{X - XVU\T}\F^2$ is minimized when $V$ corresponds to the $k$ ordinary PCs. Based on this, \citet{zou2006sparse} shows that the algorithm of SPCA searches for a sparse approximation of the ordinary PCs, yet without sparsity-enabling orthogonal rotations (i.e., it assumes row sparsity).
	
	\item[SPC \citep{witten2009penalized}] SPC finds one sparse PC at a time,
	\begin{eqnarray} \label{eqn:spc}
		\maxi_{u,v} & u_i\T X v_i &\st \norm{u_i}_2=1,~ \norm{v_i}_2=1,~\norm{v_i}_1\le\gamma,
	\end{eqnarray}
	where $v_i\in\R^p$ contains the loadings of the $i$th sparse PC, for $1\le i \le k$. 
	When $k=1$, our formulation of SCA in \eqref{eqn:sca2} takes the same form as the SPC formulation, where an orthogonal rotation is unnecessary. 
	When $k>1$, however, SPC searches for sparse PCs sequentially and does not rotate PCs, unlike SCA, which computes $k$ sparse PCs simultaneously. 
	SPC is similar to the rSVD proposed by \citet{shen2008sparse} and the TPower proposed by \citet{yuan2013truncated} in that all the three methods rely on a deflation technique for multiple PCs. 
	This technique entails complications of, for example, non-orthogonality and sub-optimality \citep{mackey2008deflation}.
	More generally, these methods can each be viewed as a special case of the following GPower formulation.
	
	\item[GPower \citep{journee2010generalized}] GPower has a ``block version'' that computes multiple sparse PCs simultaneously by considering a linear combination of individual sparse PCA (as formulated in SPC),
	\begin{eqnarray*} \label{eqn:gpower}
		\maxi_{U,V} &  \sum_{j=1}^k \mu_j u_{j}\T X v_{j} - \sum_j\lambda_j\norm{v_{j}}_1 &\st U\in\B(n,k),~ V\in\V(p,k),
	\end{eqnarray*}
	where $V$ contains the PC loadings, and $u_j$ and $v_j$ are the $j$th column of $U$ and $V$ respectively, and $\mu_j$ is the weight for the $j$th sparse PC, and $\lambda_j$ is the sparsity tuning parameter for the $j$th sparse PC.
	The algorithm of GPower fundamentally deals with sparse PCs individually, which prohibits orthogonal rotations (on $V$). 
	
	\item[SPCArt \citep{hu2016sparse}] SPCArt is the first (to our knowledge) sparse PCA method that concerns orthogonal rotations in its formulation. It searches for sparse PCs by directly approximating the singular vectors (as opposed to minimizing the reconstruction error or maximizing the explained variance),
	\begin{eqnarray*} \label{eqn:spcart}
		\mini_{Y,R} &  \norm{V - YR}^2\F+\lambda\norm{Y}_1 &\st Y\in\V(p,k),~ R\in\U(k),
	\end{eqnarray*}
	where $V\in\V(p,k)$ contains the top $k$ singular vectors of $X$, and $Y$ contains the sparse loadings.
	Conceptually, introducing an orthogonal rotation ($R$) allows a larger searching space for $Y$. 
	However, the algorithm of SPCArt does not specifically update $R$ to promote sparsity (e.g., minimize $\norm{Y}_1$ as in SCA); instead, SPCArt simply computes $R$ so as to align the polar of $V$ and $Y$ (i.e., $\hat{R}=\polar(Y\T V)$). 
	As such, the performance of SPCArt could be sensitive to the initialization of $Y$. Empirically, SPCArt yields results that are nearly comparable to the GPower based method, as concluded by the authors. 
\end{description}

\subsection{Sparse coding and independent component analysis}
Sparse coding concerns low-rank representations of individual samples. 
We view it as a variant of PCA, where we presume the component scores to be sparse. 
Recall that the scores are the representations of individual data points in $\R^k$, where $k$ is the number of PCs. 
In particular, presuming sparse scores implies that each data point is correlated with only a small subset of PCs. 
Sparse coding is useful to generate simple representations of individual date points, and the basis of such representations (i.e., PCs) usually provide scientific insights.
For example, sparse coding of natural images recovers the common understanding of how the primary visual cortex in mammalian perceives scenes (see Section \ref{sxn:coding} for an example).

The SCA algorithm can be used to solve sparse coding. 
This is because, similar to SCA, sparse coding can be viewed as a special case of the SMA problem. To see this, simply omit the sparsity constraint on $Y$ in \eqref{eqn:sma1},
\begin{eqnarray*}
	\mini_{Z,B,Y} & & \norm{X- Z B Y\T}\F \\ 
	\st & & Z\in\B(n,k),~ Y\in\B(p,k),~ \cP_1(Z) \le \gamma_z\nonumber 
\end{eqnarray*}
Here, $Z$ contains the sparse scores, and $BY\T$ contains the basis of sparse coding.
To solve sparse coding, we apply the SCA algorithm (Algorithm \ref{alg:sca}) to the transposed data matrix, $X\T$. In doing this, the output of the algorithm is actually an estimate of sparse component scores for the original data matrix. 

More broadly, independent component analysis (ICA) is widely applied for sparse coding in the signal processing literature. 
Despite the different motivations, sparse PCA on a transposed data matrix appears to perform very similarly to sparse ICA on the original data. We elaborate on this in Supplementary Section S3 
and apply SCA to blind source separation of images.

\section{Simulation studies} \label{sxn:simu}
In this section, we compare several sparse PCA methods using simulated data. 
Specifically, we focused on (1) their ability of explaining variance in the data, (2) the robustness against varying sparsity parameters, and (3) the computational speed. 
We selected SPCA, SPC, GPower, the SPCAvRP method recently proposed by \citet{gataric2020sparse}, SCA, and another variant of SCA which deploys the absmin rotation (SCA-absmin, see Remark \ref{rmk:absmin} of Section \ref{sxn:rotation}). For SCA and SCA-absmin, we implemented the algorithms in \texttt{R}.\footnote{We provide an \texttt{R} package \texttt{epca}, for \textbf{e}xploratory \textbf{p}rincipal \textbf{c}omponent \textbf{a}nalysis, which implements SCA and SMA with various algorithmic options. The package is available from CRAN (\url{https://CRAN.R-project.org/package=epca}).}
For SPCA, SPC, and SPCAvRP, we invoked the original \texttt{R} packages \texttt{elasticnet}, \texttt{PMA}, and \texttt{SPCAvRP} respectively.
The implementation of GPower (in \texttt{MATLAB}) was obtained from the authors' website.
For all the iterative methods, we specified maximum number of iterations to 1,000 and the stopping (convergence) criterion to $10^{-5}$.
Overall, our numerical experiments showed that the SCA algorithm converges faster and produces more robust sparse PCs that capture a larger amount of variance in the data.

\subsection{Proportion of variance explained} \label{sxn:simupve}
In this simulation, we compared the abilities of sparse PCA methods in explaining variance in the data. 
To this end, we simulated $30$ data matrices with $n=100$ observations and $p=100$ variables from the following low-rank generative model:
$$X=SY\T+E,$$
where $S\in\R^{100 \times 16}$ contains the component scores, and $Y\in\R^{100 \times 16}$ contains the loadings of sparse PCs, and $E\in\R^{100 \times 100}$ is some noise.
To generate $S$, we randomly sampled $U\in\V(100,16)$ and $V\in\U(16)$ and set $S=U \Sigma V\T$, where $\Sigma$ is a diagonal matrix with the diagonals $\sigma_l=10-\sqrt{l}$ for $l=1,2,...,16$.
To simulate a sparse $Y$, we took a random element from $\V(100,16)$, then soft-threshold its elements with sparsity parameter $\gamma=20$ (i.e., $T_{20}$ as defined in Equation \eqref{eqn:soft}).\footnote{We also experimented with $\gamma=\sqrt{pk}=40$. The results are comparable.}
Note that, it is unnecessary to re-scale the columns of loadings to unit length, because the column of $S$ can absorb these scalars.
Lastly, the elements in $E$ were drawn independently from the normal distribution, $E_{ij} {\sim} \text{N}(0,0.1^2)$. 

We applied the six sparse PCA methods to each simulated data matrix $X$ with $k=2,4,6,...,16$. 
For each $k$, we imposed the same $\ell_{1}$-norm constraint on the sparse loadings for all methods. Specifically, for SCA, and SPC, we directly configured the sparsity controlling parameters to $2.5k$.\footnote{The coefficient 2.5 is calculated from $\lambda/16$, assuming that the 16 sparse PCs have equally distributed $\ell_1$-norm.} As for SPCA, GPower and SPCAvRP, to ensure a fair comparison, we tuned the parameters using binary search such that the returned loadings all have the same $\ell_1$ norm of $2.5k$. 
To evaluate sparse PCs, we define the cumulative proportion of variance explained (PVE) by the first $k$ sparse PCs as $\norm{X_Y}\F^2$, where $X_Y=XY(Y\T Y)^{-1}Y\T$ \citep{shen2008sparse}. 
Note that the PVE by sparse PCs is upper bounded by that of ordinary PCs (no sparsity constraint). Therefore, we also applied PCA to $X$ for comparison.
Figure \ref{fig:4_pve} displays the mean PVE for different PCA methods, varying the requested number of PCs from 2 to 16. It can be seen that SPCAvRP and SPCA explained less than half of the PVE by PCA, and that GPower and SPC both exhibited some improvements over SPCA. For GPower, we tested both the single-unit and the block versions, but the block version often converged to a defective solution with some columns decaying to all zeros. This happened when the number of targeted PCs went above 5 in this simulation. As such, we display only the single-unit version of the results. Overall, SCA performed the best among sparse PCA methods and were the closest to PCA. 
In addition, the SCA algorithm converged with fewer iterations than the other sparse PCA methods (see Table \ref{tab:iter} for a comparison when $k=16$). 
We also observed that using the varimax rotation (SCA), the algorithm was more computationally efficient than using the absmin rotation (SCA-absmin).

\begin{figure}
	\centering
	\includegraphics[width=0.6\linewidth]{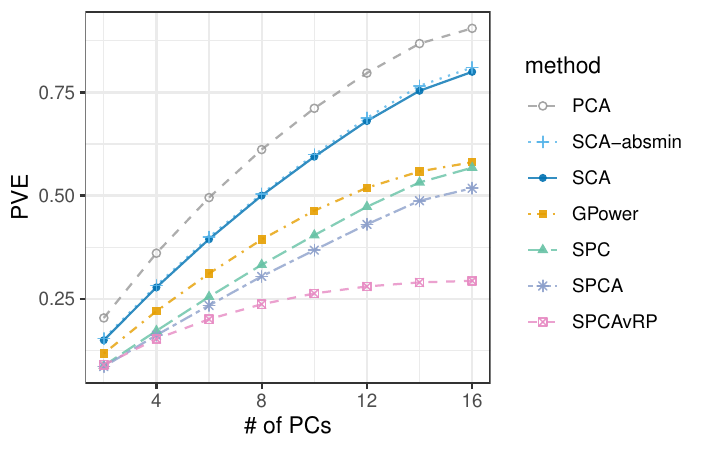}
	\caption{
		Comparisons of sparse PCA methods using simulated data.
		The proportion of variance explained (PVE) by sparse principal components (PCs) with the number of targeted PCs varying from 2 to 16. 
	}
	\label{fig:4_pve}
\end{figure}

\begin{table}
\centering
\begin{tabular}{|c|c|c|c|}
	\hline
	\textbf{Method} & \textbf{\# of iterations} & \textbf{Mean run time (s)} & \textbf{Environment} \\
	\hline
	SCA & 10 $\sim$ 65 (all PCs) & 0.96 & \texttt{R} \\
	SPC & 25 $\sim$ 1,000 (each PC) & 1.21 & \texttt{R} \\
	GPower & 30 $\sim$ 150 (each PC) & 0.19 & \texttt{MATLAB} \\
	SPCA & 470 $\sim$ 920 (all PCs) & 56.30 & \texttt{R} \\
	SPCAvRP & / & 28.67 & \texttt{R}\\
	SCA-absmin & / & 23.5 & \texttt{R} \\
	\hline
\end{tabular}
\caption{
	Comparison of the computational efficiency of sparse PCA methods. 
	Each method is tasked to find 16 PCs on a single CPU (2.50GHz).
	SPCAvRPs is not iterative (yet is parallelizable), hence the number of iterations is not applicable. 
	The absmin rotation is less efficient, so we halted the algorithm of \textsc{SCA-absmin} after the 15th iteration.
}
\label{tab:iter}
\end{table}

\subsection{Robustness against tuning parameters} \label{sxn:simusbm}

This simulation study investigates the robustness of sparse PCA to the choice of sparsity parameters. 
For this, we applied sparse PCA to detect communities in networks (or graph partitioning) \citep[see, e.g.,][]{fortunato2010community}, using the graph adjacency matrix (see the definition below) as input. 
This application is possible thanks to the recent consistency results \citep{rohe2020vintage} showing that under the stochastic block model \citep[SBM, see for example][]{holland1983stochastic}, the support of each sparse PC estimates the membership (indicator) of one community. 
Hence, we could evaluate sparse PCs by examining their support. 

We simulated 30 undirected graphs with $n=900$ nodes and four equally sized blocks from the SBM. Under the SBM, the edge between node $i$ and $j$ is sampled from the Bernoulli distribution, $\text{Bernoulli}(B_{z(i),z(j)})$, where $z(i)\in\{1,2,3,4\}$ is the membership of node $i$, and 
$$B=0.05\times\begin{bmatrix}
	0.6& 0.2& 0.1& 0.1\\
	0.2& 0.7& 0.05& 0.05\\
	0.1& 0.05& 0.6& 0.25\\
	0.1& 0.05& 0.25& 0.6
\end{bmatrix}$$
is the block connectivity matrix. Under this setting, the expected number of edges connected to each node is 45. 
For each simulated graph, we defined the adjacency matrix $A\in\{0,1\}^{n \times n}$ with $A_{ij}=1$ if and only if $i$ and $j$ are connected. 

We applied SCA, SPC, and GPower\footnote{Since SPCA and SPCAvRP performs worse than SPC and GPower \citep{zou2018selective}, we excluded the two methods in this simulation for simplicity.} to each of the 30 simulated adjacency matrices with $k=4$. 
We varied the sparsity parameter $\gamma$ to take value in $\{18, 24, 36, 48, 60, 66\}$.
For SPC, we required each of the four PCs to have $\ell_1$ norm $\gamma/4$. 
As for GPower, we tuned its parameters such that the returned loading matrix has the $\ell_1$ norm of $\gamma$. 
Figure \ref{fig:4_heatmap} depicts the estimated loadings returned by SCA and SPC. On the left two columns of panels ($\gamma=48$ and $36$), the supports of the four sparse PCs were well separated and indicated block memberships. This suggested that we could use the loadings to cluster nodes and quantitatively assessed the quality of sparse PCA methods. 
Specifically, we assigned node $i$ to cluster $j$ if $Y_{ij}$ is the largest absolute value in the $i$th row of $Y$, that is, $|Y_{ij}|>|Y_{il}|$ for all $l \neq j$. In the case of ties or all-zero rows, the cluster label is randomly assigned. For each estimate, let $C\in\{1,2,3,4\}^n$ contain the assigned cluster labels and $C^*\in\{1,2,3,4\}^n$ contain the true labels. Define the accuracy as 
$$\text{Accuracy}(C, C^*)=\max_{\pi\in\mathcal{P}(4)}\left\{\frac{1}{n}\sum_{i=1}^n\1\left(\pi\left(C_i\right)= C^*_i\right)\right\},$$
where $\mathcal{P}(4)$ contains all the possible permutation functions of the set $\{1,2,3,4\}$, and $\1(x)$ is the indicator function of $x$.
We used the accuracy to assess the quality of the sparse PCA solutions. 
Figure \ref{fig:4_mcr} depicts the accuracy of the three methods with varying sparsity parameters. 
It can be seen that the performance of GPower and SCA were less affected by the changing of sparsity parameter, while SPC was profoundly influenced.
As $\gamma$ became smaller, SPC quickly lost its power in community detection, suggesting that SPC is more sensitive to the choices of tuning parameter. 
Although less sensitive to the change in $\gamma$, GPower produced poorer estimation of sparse PCs, with the accuracy slightly better than random guesses ($\text{accuracy}=0.25$). 
Overall, SCA yielded higher accuracy with smaller deviation compared to the others, suggesting that SCA is less dependent on the choice of sparsity parameters.

\begin{figure}
	\centering
	\includegraphics[width=0.66\linewidth]{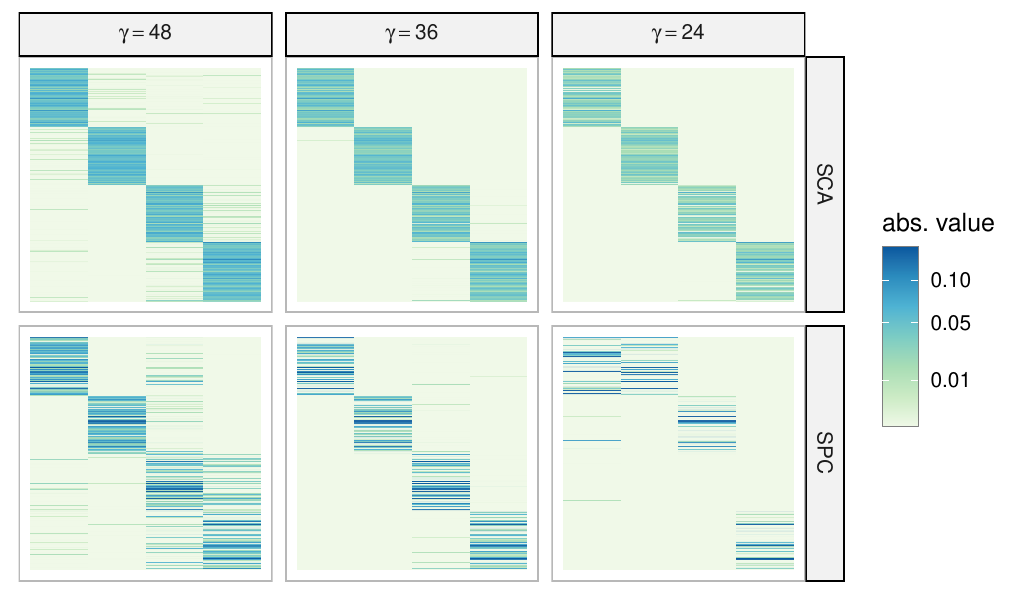}
	\caption{Comparisons of SCA and SPC using simulated network data. Heat maps of the loadings ($900\times4$ matrices) returned by SCA and SPC using three different sparsity parameters ($\gamma=24,36,48$). In each heat map, rows correspond to nodes, which are grouped by the true community membership, and each column corresponds to one sparse PC. The color shade indicates the absolute of loadings.}
	\label{fig:4_heatmap}
\end{figure}

\begin{figure}
	\centering
	\includegraphics[width=.5\linewidth]{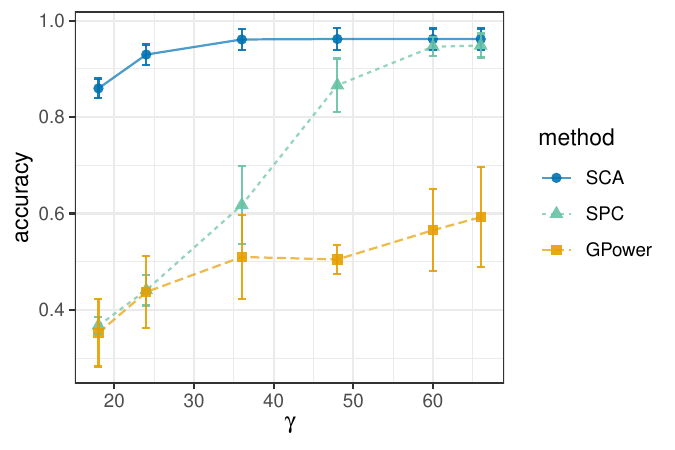}
	\caption{
		Comparisons of sparse PCA methods using simulated network data. The accuracy of SCA, GPower, and SPC in community detection using various sparsity parameters ($\gamma$). Each point indicates the mean accuracy across 30 replicates, and the error bar indicates the standard deviation of the evaluated accuracy. 
	}
	\label{fig:4_mcr}
\end{figure}

\begin{table}
	\centering
	\begin{tabular}{|c|c|c|c|c|c|c|}
		\hline
		& $\gamma=18$ & $\gamma=24$ & $\gamma=36$ & $\gamma=48$ & $\gamma=60$ & $\gamma=66$ \\ 
		\hline
		Using SCA solution & 191.47 & 323.36 & \textbf{1135.03} & \textbf{1906.25} & \textbf{2554.86} & \textbf{2783.73} \\
		Using SPC solution & \textbf{544.81} & \textbf{705.01} & 1029.04 & 1195.91 & 1334.67 & 1423.95 \\
		\hline
	\end{tabular}
	\caption{Comparison of the SPC objective values, $\sum_{i=1}^4(u_i\T A v_i)^2$ (see Equation \eqref{eqn:spc}), evaluated using the output of the SCA and SPC algorithms with various sparsity parameter ($\gamma$).}
	\label{tab:obj}
\end{table}

In this example, SCA outperforms SPC because it finds a better optimization solution.  This comparison could be made difficult by the fact that they have different objective functions.  However, in this case, even though SCA is optimizing a different objective function, it outperforms SPC at \textit{optimizing the SPC objective function}. 
Table \ref{tab:obj} lists the objective values of SPC (Equations \eqref{eqn:spc}) evaluated using the solutions of the SCA and SPC algorithms with various $\gamma$. When $\gamma\in\{36, 48, 60, 66\}$, the SCA algorithm outputs a solution that achieves a higher value of the SPC objective, suggesting that the SPC algorithm is likely to return local optima.

\section{Applications} \label{sxn:app}
In this section, we applied SCA to real data. 
The first application is the sparse coding of natural images. It illustrates the utility of sparse PCA as independent component analysis. Supplementary Section S3.1 
contains another application of SCA to blind source separation of images. 
Next, we demonstrate the ability of SCA in handling high-dimensional problems (i.e., $p > n$) through a transcriptome sequencing dataset and a targeted sample of Twitter friendship network. These datasets are of large scale. To our knowledge, no other current implementations of sparse PCA can efficiently handle a large matrix at the scale. As such, we will restrict our discussion to SCA. 

\subsection{Sparse coding of images} \label{sxn:coding}

\begin{figure}	
	\centering
	\begin{multicols}{2}
		\textbf{PCA} \\
		\includegraphics[width=.75\linewidth]{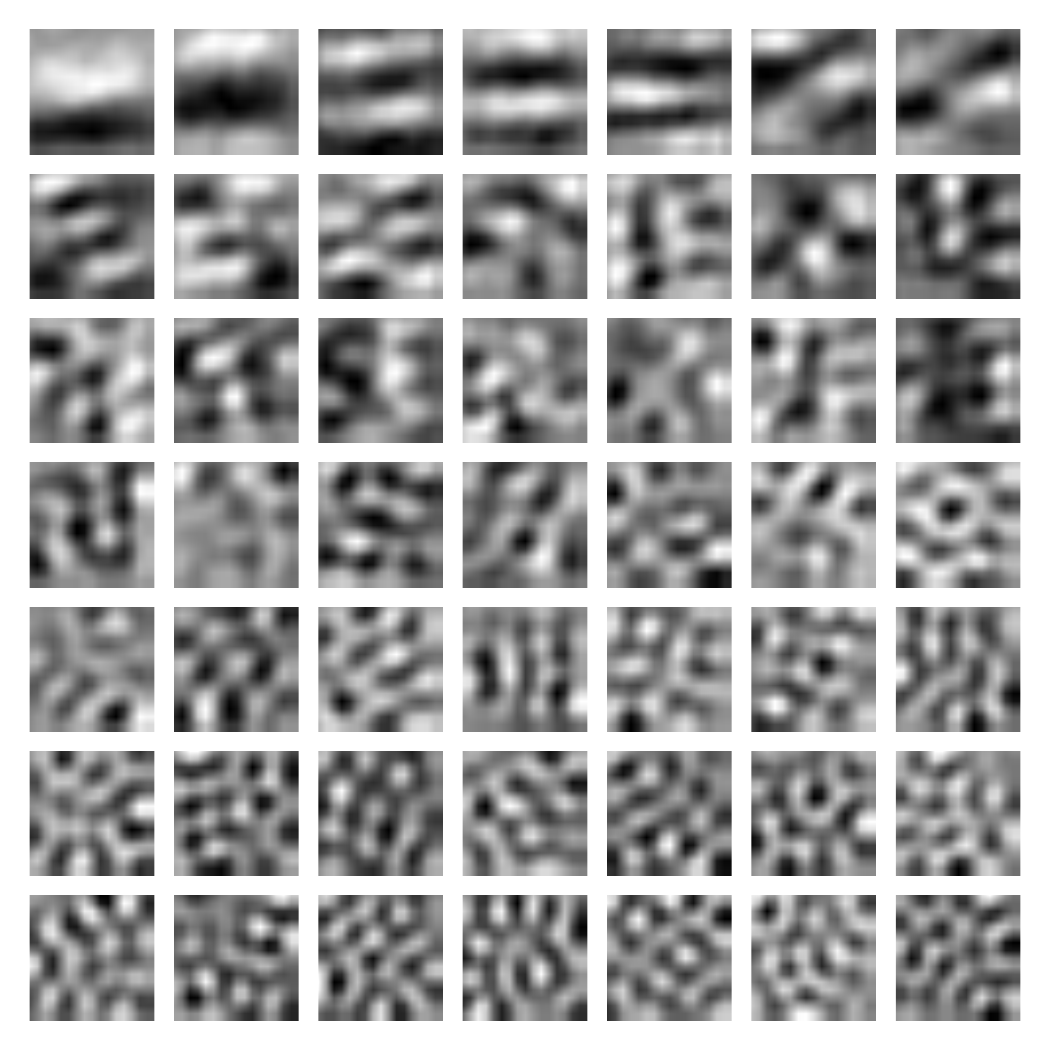} \\
		\textbf{SCA} \\
		\includegraphics[width=.75\linewidth]{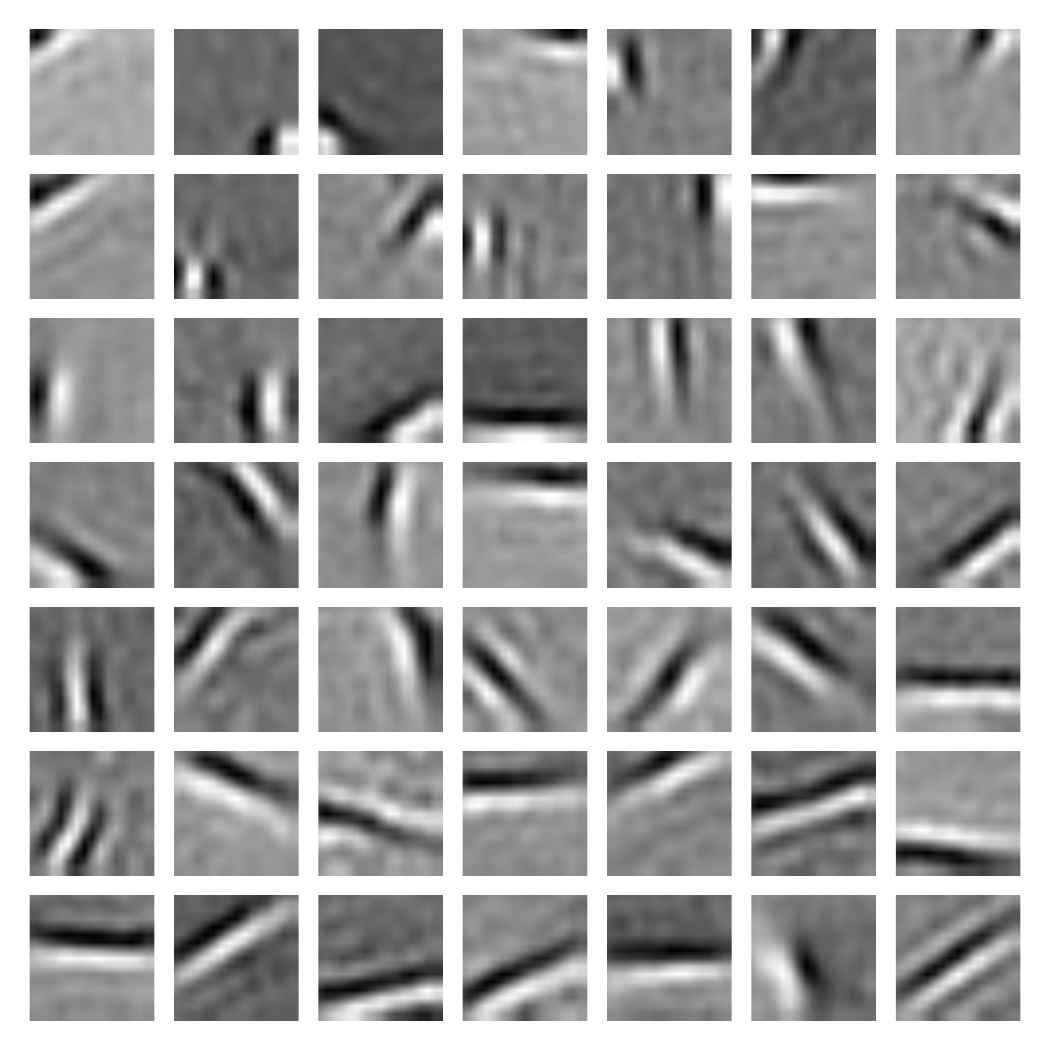}
	\end{multicols} 
	\vspace{-1.2em} 
	\caption{Sparse image encoding using PCA (left) and SCA (right). For both method, shown are the $49$ image bases (i.e., component scores) extracted from natural images. Each image basis is in $16\times16$ pixel.}
	\label{fig:5_encoding}
\end{figure}

Low-level visual layers, such as retina, the lateral geniculate nucleus, and the primary visual cortex (V1) are shared processing components in mammalian. 
The receptive fields in the V1 can be characterized as being spatially localized, oriented and bandpass (i.e., selective to structure at different spatial scales).
To understand V1, one line of research focuses on finding sparse and linearly independent codes for natural images, which provides an efficient representation for later stages of processing \citep{field1994goal, olshausen1996emergence, bell1997independent}. 
This type of research is based on the hypothesis of sparse coding, that is, any perceived scenes can be synthesized via the linear combination of some small subsets of basis images \citep{lee2006efficient, gregor2010learning}).
In this application, we show that sparse PCA produces a set of bases for natural images that resembles those found in \citet{olshausen1996emergence}. 

We utilized ten natural images from \citet{olshausen1996emergence}, each of which contains $512 \times 512$ pixels. 
We followed the same whitening process as described by the authors. 
Next, we randomly sampled a total of $12,000$ small image patches the ten images, where each patch contains $16\times16$ pixels. 
This was followed by a centering step that subtracts each pixel by the mean of all $256$ pixels. 
We vectorized each patch of image and put them into the rows of a data matrix, $X\in\R^{n \times p}$, where $n=12,000$ and $p=256$.
Finally, we applied SCA to the transposed data matrix, $X\T$ (note that this is sparse coding). For this exploratory analysis, we set $k=49$ to find 49 sparse PCs (the same result holds for various selections of $k$) with the default sparsity parameter, $\gamma=\sqrt{pk}$.
In particular, for the varimax rotation, we normalized the rows to unit length rescaled them afterward, as recommended by \citet{kaiser1958varimax}. 
In the output of SCA, the estimated scores $S\in\R^{p \times k}$ contains the basis images, and the estimated sparse loadings $Y\in\R^{n\times k}$ encodes how the basis images are linearly combined to form each image patch (i.e., $Y$ contains the linear coefficients). 

Figure \ref{fig:5_encoding} displays the 49 image bases returned by PCA and SCA, where each image represents one column of $S$ (transformed into a $16\times16$ array). 
For SCA, all of the basis images appeared to exhibit simple patterns, such as lines and edges.
As for PCA, the oriented structure in the first few basis images does not arise as a result of the oriented structures in natural images, yet more likely because of the existence of those components with low spatial frequency \citep{field1987relations}.

\subsection{Analysis of single-cell gene expression data}\label{sxn:scrnaseq}

Single-cell transcriptome sequencing (scRNA-seq) provides high-throughput transcriptome expression quantification at individual cell level. It has been widely used across biological disciplines. For example, patterns of gene expression can be identified through clustering analysis. This helps uncover the existence of rare cell types within a cell population that have never been seen \citep{plasschaert2018single,montoro2018revised}. In this application, we aimed to use SCA to extract the sparse PCs of genes that characterize some known cell types.

\begin{figure}
	\centering
	\includegraphics[width=0.55\linewidth]{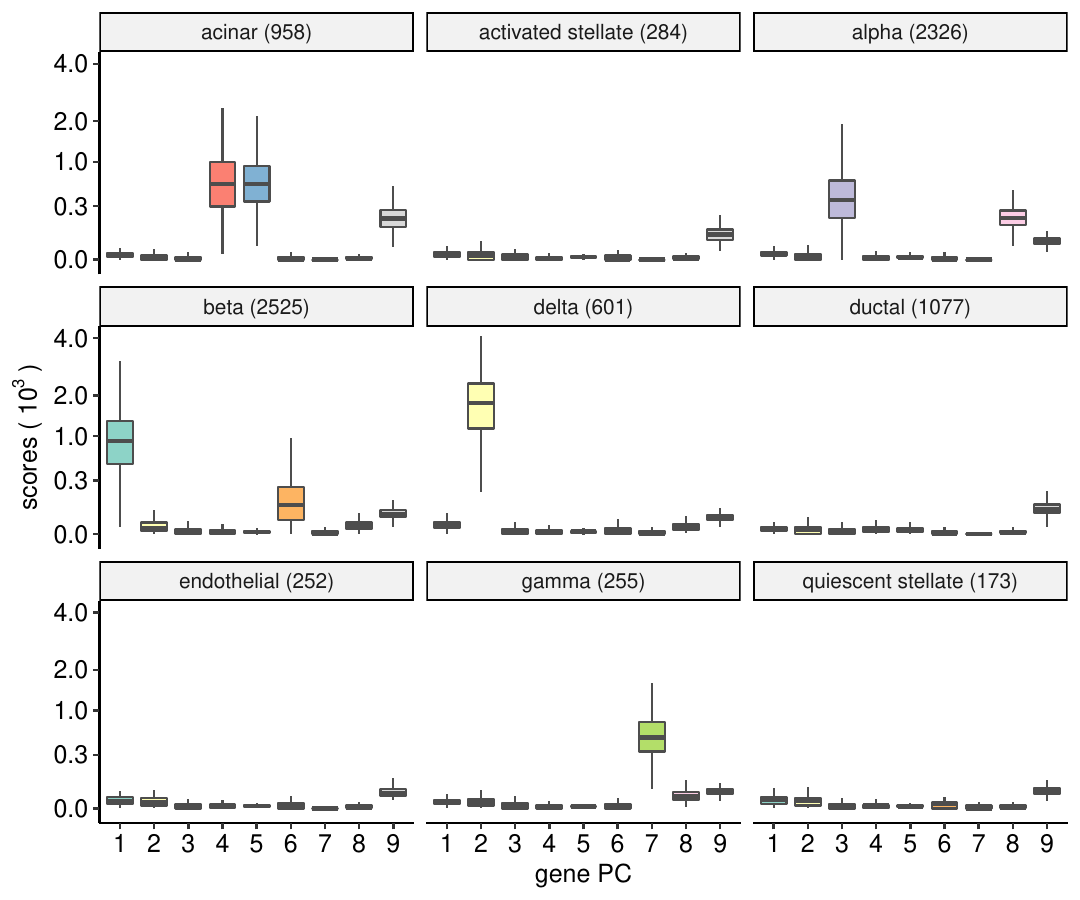}
	\caption{Scores of sparse gene principal components (PCs) stratified by cell types. Each panel displays one of nine cell types with the names of cell types and the number of cells reported on the top strips. For each cell type, a box depicts the component scores for nine sparse gene PCs.}
	\label{fig:6_scRNAseq}
\end{figure}

\begin{table}[!b]
	\centering
	\begin{tabular}{|c|c|l|}
		\hline
		\textbf{PC} & \textbf{\# of genes} & \textbf{Gene name(s)} \\ 
		\hline
		1 &   1 & INS  \\ 
		2 &   1 & SST  \\ 
		3 &   1 & GCG  \\ 
		4 &   8 & CTRB2, REG1A, REG1B, REG3A, SPINK1 ... \\ 
		5 &  15 & CELA3A, CPA1, CTRB1, PRSS1, PRSS2 ... \\ 
		6 &   1 & IAPP  \\ 
		7 &   1 & PPY  \\ 
		8 &   3 & CLU, GNAS, TTR  \\ 
		9 &  61 & ACTG1, EEF1A1, FTH1, FTL, TMSB4X ... \\  
		\hline
	\end{tabular}
	\caption{Sparse gene PCs estimated by SCA. For each gene PC, the number of genes (i.e., the number of non-zeros in the loadings) and the top 5 genes according to the absolute loadings are reported.}
	\label{tab:genePC}
\end{table}

For this application, we used the human pancreatic islet cell data from \citet{baron2016single}. We removed the genes that do not exhibit variation across all cells (i.e., zero standard deviation) and removed the cell types that contain fewer than 100 cells. This resulted in a data matrix $X\in\R^{n \times p}$ of $n=8,451$ cells across nine cell types and $p=17,499$ genes, with $X_{ij}$ measuring the expression level of gene $j$ in cell $i$. $X$ is sparse; it contains 10.8\% non-zero elements. 
We applied SCA on $X$ to find $k=9$ sparse gene PCs. We set the sparsity parameter to $\gamma=\log(pk)\approx12$, as we aimed for particularly sparse PCs (i.e., each PC is consist of a small number of genes). The algorithm took about 5 minutes (24 iterations) to complete on a single processor (3.3GHz). 
As a result, each column of the loading matrix contains a small number of non-zero elements, suggesting that most of the gene PCs consist of one or a few genes. Table \ref{tab:genePC} lists the names of these genes for each PCs. For example, the PC 2 consists of only one gene, SST.
Despite the simple structure of PCs, these PCs picked up informative gene markers for individual cell types. To see this, we calculated the scores for each cell using the 9 PCs (That is, each cell gets 9 scores, each of which corresponds to one of the nine PCs). Figure \ref{fig:6_scRNAseq} displays the box plots of the scores stratified by cell type.
For example, the expression of the SST gene (which solely composes the 2nd PC) identifies the ``delta'' cells. This result highlights the power of scRNA-seq in capturing cell-type specific information and suggests the applicability of our methods to high-dimensional biological data.

\subsection{Clustering of Twitter friendship network} \label{sxn:twitter}

This application serves in a grand efforts of ours to study political communication on social media, like Twitter. The information on Twitter is organized so that users primarily read the tweets of their ``friends.'' In order to select content, a user can freely ``follow'' (and ``unfollow'') any other accounts, and we call these other accounts the friends of it. Thanks to this design, the communication on Twitter can be contextualized by the friendship network. As such, we hypothesize that user's community membership in the network offers the context of user's opinion expression on social media \citep{zhang2022social,zhang2022network,zhang2022peripheral}. To study the hypothesis, a key step is to cluster Twitter accounts using their friendship network. In this section, we demonstrate large-scale network clustering using sparse PCA. 

\begin{figure}[!h]
	\centering
	\includegraphics[width=.72\linewidth]{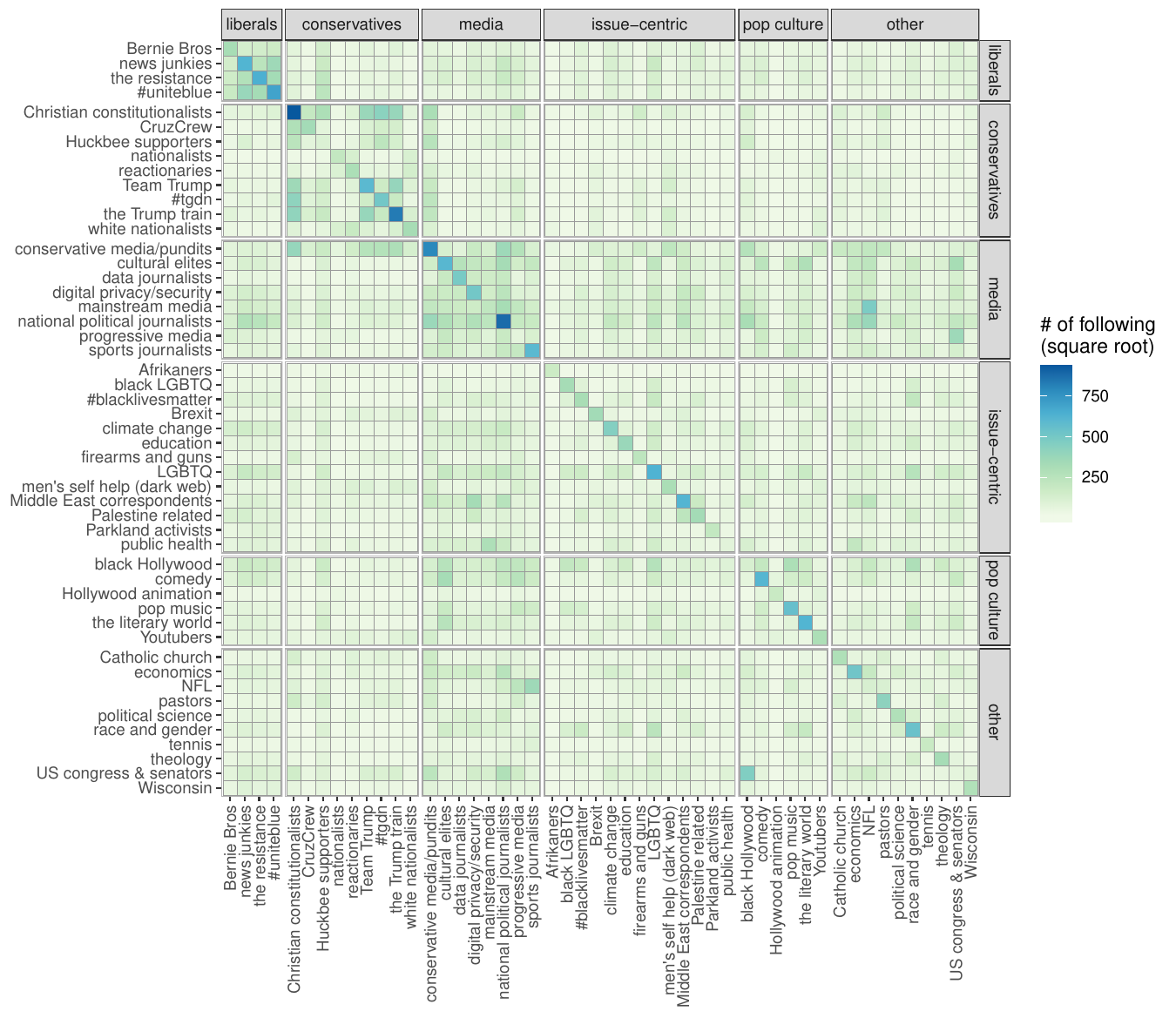}
	\caption{Heat map of friend counts between row and column clusters of Twitter accounts. Each row and column corresponds to a cluster. The row and column panels indicate cluster category, with the category names shown in the top and right strips. The color shades indicate the number of followings from the row cluster to the column cluster, after the square root transformation.}
	\label{fig:7_twitter_followings}
\end{figure}

For this application, we collected a targeted sample from the Twitter friendship network in August 2018 \citep{chen2020targeted}. In this sample, there are $n=193,120$ Twitter accounts who follow a total of $p=1,310,051$ accounts, after filtering out the accounts with few followers or followings. We defined the graph adjacency matrix $A\in\{0, 1\}^{n \times p}$ with $A_{ij} = 1$ if and only if account $i$ follows account $j$.\footnote{The columns of $A$ are not centered nor scaled. One alternative is to use the normalized version of $A$. For example, define the regularized graph Laplacian as $L\in\R^{n \times p}$ with $L_{ij}=A_{ij}/\sqrt{(r_i + \bar r) (c_j + \bar c)}$, where $r_i=\sum_jA_{ij}$ is the sum of the $i$th row of $A$, $c_j=\sum_iA_{ij}$ is the sum the $j$th column of $A$. Here, $\bar r$ and $\bar c$ are the means of $r_i$'s and $c_j$'s respectively. \citep{zhang2018understanding}.} 
This resulted in a sparse $A$ with about 0.02\% entries being 1. 
We applied SMA to $A$ with $k=100$ and default sparsity parameters. 
This analysis was computationally tractable; one iteration of the SMA algorithm took about 54 minutes on a single processor (2.5GHz), thanks to the efficient algorithm that computes the sparse SVD \citep{baglama2005augmented}. 
Figure \ref{fig:2_rotation} displays seven example columns of $Y$. 
Using the output $Z\in\R^{n\times k}$ and $Y\in\R^{p\times k}$ from SMA, the clusters of Twitter accounts were determined as follows (same as in Section \ref{sxn:simusbm}): the $i$th row account of $A$ was assigned to the $l$th row cluster if $Z_{il}$ was the greatest in the $i$th row of $Z$, that is, $|Z_{il}|\ge |Z_{il'}|$ for all $l'=1,2,...,k$, and the $j$th column account of $A$ was assigned to the $l$th column cluster if $Y_{jl}$ was the greatest in the $j$th row of $Y$, $|Y_{jl}|\ge |Y_{jl'}|$ for all $l'=1,2,...,k$. 
Upon detailed evaluation of these clusters, we showed that our clustering of Twitter accounts formed homogeneous, connected, and stable social groups \citep{zhang2022social}. 
For example, we found that a user is more likely to retweet the content that originated from another member in the same clusters (p-value $<10^{-16}$ in a $\chi^2$ test).
More interestingly, the estimated row clusters and column clusters are matched \citep{rohe2016co}, that is, the $k$th row cluster tends to follow the accounts in the $k$th column cluster. To illustrate this, we quantified the number of followings from the row clusters to the corresponding column clusters. Figure \ref{fig:7_twitter_followings} displays the results for 50 selected clusters that are related to U.S. politics. It can be seen that the number of followings between each paired row and column clusters (i.e., the diagonals in Figure \ref{fig:7_twitter_followings}) showed marked enrichment. 
These results suggest the efficacy of our methods for analysis of social network data.

\section{Discussions} \label{sxn:disc}

In this paper, we introduced SCA, a new method for sparse PCA, and SMA, an extension for two-way matrix analysis. 
SCA differs from the existing sparse PCA methods in that it estimates column sparse PCs, that is PCs that are sparse in an orthogonally rotated basis. 
This is particularly useful when the singular vectors of a data matrix (or the eigenvectors of the covariance matrix) are not readily sparse.
We demonstrated that it explains more variance in the data than the state-of-the-art methods of sparse PCA. 
In addition, the algorithm is also stable and robust against a wide choices of tuning parameters.
In practice, SCA is advantageous when multiple PCs are desired because it does not require the deflation.

\section*{Acknowledgments}
We thank S\"{u}nd\"{u}z Kele\c{s}, S\'{e}bastien Roch, Po-Ling Loh, Michael A Newton, Yini Zhang, Muzhe Zeng, Alex Hayes, E Auden Krauska, Jocelyn Ostrowski, Daniel Conn, and Shan Lu for all the helpful discussions.

\newpage
\appendix
\begin{center}
{\LARGE\bf Appendices}
\end{center}

\section{Technical proofs} \label{sxn:proof}


\begin{proof}\textbf{of Proposition \ref{thm:improve}}
	Since $D$ is diagonal, and $B$ can be any matrix including diagonal, the inequality result holds. 
	Furthermore, given any fixed $Z$ and $Y$ subject to the constraints in \eqref{eqn:sca1} (i.e., $Y$'s columns are not the leading eigenvectors of $X$), the maximizer on the left-hand-side is $B^*= Z\T X Y$ which is not diagonal.\footnote{Generally, $B^*=\left(Z\T Z\right)^{-1} Z\T X Y \left(Y\T Y\right)^{-1}$ if $Z$ and $Y$ are full-rank, or $B^*=\left(Z\T Z\right)^{+} Z\T X Y \left(Y\T Y\right)^{+}$ if either $Z$ or $Y$ is singular, where $A^+$ is the Moore--Penrose inverse of matrix $A$.} Hence, the inequality is strict. 
\end{proof}

\begin{proof}\textbf{of Lemma \ref{lem:translate1}}
	We rewrite the objective function:
	\begin{eqnarray*}
		\norm{X-ZBY\T}^2\F  &=& \tr\left[\left(X-ZBY\T\right)\T\left(X-ZBY\T\right)\right]\\
		&=&\norm{X}^2\F -2\tr\left(X\T ZBY\T\right)+\tr\left(B\T B\right)\\
		&=&\norm{X}^2\F -\tr\left[B\T\left(2Z\T XY-B\right)\right].
	\end{eqnarray*}
	For fixed $Z$ and $Y$, take the derivative of $B$ and set it to zero. We have the optimizer
	$B^*=Z\T XY$
	and the squared optimal value is $\norm{X}^2\F -\norm{Z\T XY}^2\F $. Recognizing that $\norm{X}^2\F$ is determined, the desired formulation \eqref{eqn:sma2} follows.
\end{proof}

\begin{remark} [Minimal matrix reconstruction error of PMD] \label{rem:pmd}
	If $B$ is constrained to a diagonal matrix in \eqref{eqn:sma1}, then the squared minimal value is 
	$\norm{X}^2\F -\sum_{i=1}^{k}d_i^2$,
	where $d_i=\left[Z\T XY\right]_{ii}$ for $i=1,2,...,k$.
\end{remark}
\begin{proof}
	From the proof of Lemma \ref{lem:translate1}, we have
	\begin{eqnarray*}
		\norm{ X-ZDY\T }^2\F 
		&=&\norm{X}^2\F -\tr\left[D\T\left(2Z\T XY-D\right)\right].
	\end{eqnarray*}
	Then, take the derivative of $D$ and set it to zero. This yields the solution $\hat{D}=\text{diag}(d_i)$, where $d_i=\left[U\T XV\right]_{ii}$.
	Finally, plugging-in the maximizer $\hat{D}$ gives the claimed optimal value. Note that $\sum_{i=1}^{k}d_i^2\le\norm{U\T XV}^2\F$.
\end{proof}

\begin{proof}\textbf{of Lemma \ref{lem:unconstrained}}
	Suppose the low-rank SVD of $C\in\R^{p \times k}$ is $UDV\T$, where $U\in\V(p,k)$, $V\in\U(k)$, and $D\in\R^{k \times k}$ is diagonal. Then, 
	$$\norm{C\T X}\F^2=\tr\left(X\T CC\T X\right)=\tr\left(X\T UD^2U\T X \right).$$
	The trace quadratic form is maximized at $X^*=UR$, for any orthogonal matrix $R\in\U(k)$.
	In particular, when $R=V$, $X^*=\polar(C)$.
\end{proof}

\section{Choosing the sparsity parameter} \label{sxn:tuning}
The sparsity controlling parameters in SCA and SMA---$\gamma$, $\gamma_y$, and $\gamma_z$---are meaningful if they take values from a certain range, depending on the choice of $\ell_p$-norm constraint. 
In this section, we discuss the sparsity constraint on $Y$; the constraint on $Z$ is similar. 
First, consider the $\ell_{1}$-norm constraint $\norm{Y}_1\le\gamma$. 
The sparsity parameter should satisfy $k \le\gamma\le k\sqrt{p}$. 
This is for the set $\{Y\in\R^{p \times k}\mid\norm{Y}_1=\gamma\}$ to intersect with $\V(p,k)$. 
On the right hand side, if $\gamma>k\sqrt{p}$, any element in $\V(p,k)$ satisfies $\norm{Y}_1<\gamma$, so the sparsity constraint is ineffective (Figure \ref{fig:3_penalty}: left panel). 
On the left hand side, if $\gamma<k$, none of the elements in $\V(p,k)$ satisfies $\norm{Y}_1\le\gamma$, so the solution to \eqref{eqn:sca3} does not fall on $\V(p,k)$.
Similarly, for the $\ell_{4/3}$-norm sparsity constraint $\norm{Y}_{4/3}\le\gamma$, the sparsity controlling parameter should take value within $k^{3/4}\le\gamma\le p^{1/4}k^{3/4}$ (Figure \ref{fig:3_penalty}: right panel).

\begin{figure}
	\centering
	\includegraphics[width=.6\linewidth]{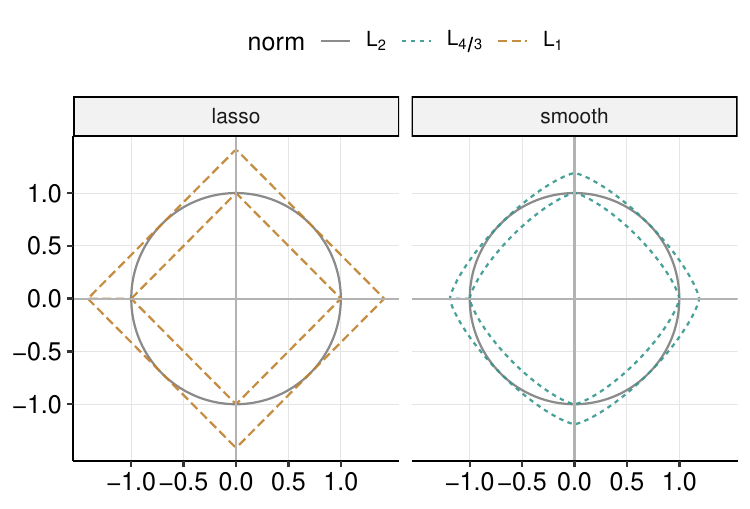}
	\caption{Comparison of the $\ell_p$ norms. Left (lasso): Two $\ell_{1}$-norm contours (brown) of 1 and $\sqrt{2}$ and the $\ell_2$-norm contour (grey) of 1. Right (smooth): $\ell_{4/3}$-norm contours (green) of 1 and $2^{1/4}$ and the $\ell_2$-norm contour (grey) of 1. }
	\label{fig:3_penalty}
\end{figure}

In Algorithm \ref{alg:sca}, the sparsity parameter is optional. 
If absent, the algorithm uses a default value of $\gamma=\sqrt{pk}$ (or $\gamma_z=\sqrt{nk}$ and $\gamma_y=\sqrt{pk}$ in SMA). 
This is supported by our simulation results showing that the SCA algorithm is robust against various choices of $\gamma$ (Section \ref{sxn:simusbm} of the paper). 
In addition, we observed that the default settings generally yielded meaningful estimates in real data applications. 

In practice, users are encouraged to experiment various values of $\gamma$ (e.g., using a grid search) to select the best fit that meets their expectations of sparsity.

The sparsity parameter can also be tuned based on the data. 
We provide a schema for cross-validate the parameters of SCA and SMA (e.g., the approximation rank $k$ and the sparsity parameter $\gamma$).
To assess a candidate parameter, we adapt a $K$-fold cross-validation framework ($K$ often takes the value 10) as previously introduced by \citet{wold1978cross}: 
\begin{enumerate} [(a)]
	\item 
	Given the input data $X\in\R^{n\times p}$, we first construct $K$ leave-out data matrices $X^{(1)}$, $X^{(2)}$, $...$, $X^{(K)}\in\R^{n\times p}$, each of which has one-$K$th disjoint portion of elements being randomly sampled and removed (i.e., set to zero). 
	Let $C^{(k)}$ collects the indices of those left-out elements in $X^{(k)}$, for $k=1,2,...,K$. 
	\item Next, we apply SCA (or the SMA) to every new matrix $X^{(k)}$ with the candidate tuning parameters and obtain its low-rank approximation $\hat{X}^{(k)}$. That is, for SCA, $\hat{X}^{(k)}=X^{(K)}\hat{Y}^{(k)}[\hat{Y}^{(k)}]\T$, and for SMA, $\hat{X}^{(k)}=\hat{Z}^{(k)}\hat{B}^{(k)}[\hat{Y}^{(k)}]\T$
	\item Finally, calculate the mean square error (MSE) of $\hat{X}^{(k)}$ over those left-out elements $C^{(k)}$, defined as
	$$\text{MSE}(k)=\sum_{(i,j)\in C^{(k)}}\left(\hat{X}^{(k)}_{ij}-X_{ij}\right)^2, k=1,2,...,K.$$
	We then evaluate the ``goodness'' of a candidate parameter by the average MSE across $K$ leave-out data matrices.
\end{enumerate} 

Upon the construction of leave-out data matrices, the left-out elements are randomly sampled; this typically removes scattered entries of $X$, rather than trunks of adjacent ones. For example, if $X$ is the adjacency matrix of a graph, then this procedure is akin to the edge cross-validation studied by \citet{li2020network}.
Setting the left-out elements to zero eliminates all terms in $\norm{Z\T X Y}\F$ that related to them. Our low-rank estimation for the missing entries is closely related to the SVD-based methods in data imputation literature \citep{troyanskaya2001missing}.

\section{Properties of soft-thresholding} \label{sxn:soft}

In the PRS update, the last step uses a shrinkage operator to project the rotated matrices onto the feasible set. 
Shrinkage operators are widely used for creating sparse structure, as it is easy to implement. The threshold value $t$ can be found in $\order(\log_2(1/\varepsilon))$ time through a binary search, where $\varepsilon$ is the convergence tolerance.

For the $\ell_1$-norm constraint (or penalty), we show that a soft-thresholding shrinkage is ``appropriate.'' 
Let $Y\in\V(p,k)$ and $\hat{Y}\in\B(p,k)$ be the two matrices before and after a shrinkage operation respectively.
A direct calculation shows that given a constraint ${\|\hat{Y}\|}_1\le\gamma$, the soft-thresholding shrinkage, $\hat{Y}=T_\gamma(Y)$, minimizes ${\|\hat{Y}-Y\|}\F$. 
After the shrinkage, the objective value in \eqref{eqn:sca3} (i.e., explained variance) decreases by at most ${\|Z\T X\|}\F{\|\hat{Y}-Y\|}\F$.
Note that we update $Y$ fixing $Z$ (and $X$).

We provide theoretical properties for the soft-thresholding, regarding preservation of orthogonality and the explained variance. 
Let $Y\in\V(p,k)$ and let $\hat{Y}=T_\gamma(Y)$ be the result of soft-thresholding $Y$ as defined in \eqref{eqn:soft}. 

First, we denote the included angles between any two columns of $\hat{Y}$ and $Y$ as $\theta_{ij}$, for $i,j=1,2,...,k$. 
When it is clear, we also write $\theta_{ii}$ as $\theta_{i}$ for simplicity.
We define the \textit{deviation} between $\hat{Y}$ and $Y$ as $\sum_{i=1}^k \sin^2(\theta_{i})$.
The following proposition bounds the sum of deviations.
\begin{proposition} [Deviation due to soft-thresholding] \label{thm:soft1}
	If $t$ is sufficiently small, then
	\begin{eqnarray*}
		\sum_{j=1}^{k}\sin^2(\theta_j)\le\norm{\hat{Y}-Y}\F^2.
	\end{eqnarray*}
\end{proposition}
\begin{proof} 
	Let $\hat{y}_i$ and $y_i$ be the $i$th column of $\hat{Y}$ and $Y$ respectively. For the included angle $\theta_i$, 
	\begin{eqnarray*}
		\cos(\theta_i) &=& \hat{y}_i\T y_i / \norm{\hat{y}}_2\\
		&=& \norm{\hat{y}_i}_2 +  \hat{y}_i\T (y_i- \hat{y}_i) / \norm{\hat{y}}_2\\
		&>&\norm{\hat{y}_i}_2.
	\end{eqnarray*}
	The last inequality results from the definition of soft-thresholding. Then, by the Pythagorean trigonometric identity, we have
	\begin{eqnarray*}
		\sin^2(\theta_i)&=&1-\cos^2(\theta_i)\\
		&<&1-\norm{\hat{y}_i}_2^2\\
		&\le&\norm{\hat{y}_i - y_i}_2^2.
	\end{eqnarray*}
	The last inequality is due to the triangular inequality.
	Finally, summing over the columns yields the desired result.
\end{proof}
Proposition \ref{thm:soft1} controls the deviation with the Frobenius norm of $Y-\hat{Y}$.
Since the columns of $Y$ are mutually orthogonal, for any two columns of $\hat{Y}$, we have 
$$\left|\hat{y}_i\T \hat{y}_j\right| \le \sin\left(\theta_j+\theta_l\right) \norm{\hat{y}_i}_2 \norm{\hat{y}_j}_2$$
assuming $\theta_i+\theta_j\le\pi/2$.
Hence, a small deviation indicates that the orthogonality of $\hat{Y}$ is conserved after soft-thresholding. 

Next, we investigate the change in explained variation due to soft-thresholding. Define the explained variance (EV) of a data matrix $X$ by the loading matrix $Y$ as $\ev(Y)=\norm{XY}\F$. 
The following proposition bounds the EV for $\hat{Y}$ and is due to the Theorem 13 in \citet{hu2016sparse}.
\begin{proposition} [Explained variance after soft-thresholding] \label{thm:soft2}
	If for all $1 \le i \le k$, $\theta_{i}=\theta$ and $\sum_{j=1}^k \cos(\theta_{ij})\le1$, then 
	$$\left(\cos^2\theta-\sqrt{k-1}\sin2\theta\right)\ev(Y)\le\ev(\hat{Y})$$ 
	for any data matrix $X$.
\end{proposition}
Proposition \ref{thm:soft2} implies that if the deviation between $Y$ and $\hat{Y}$ is small, then the EV of $\hat{Y}$ is close to that of $Y$, $$\left(\cos^2\theta-\order(\theta)\right)\ev(Y)\le\ev(\hat{Y}).$$

\section{Independent component analysis} \label{sxn:ica}
In this section, we demonstrate the connection between sparse PCA (specifically, our SCA formulation) and independent component analysis (ICA).

ICA is motivated by blind-source (or blind-signal) separation in signal processing \citep[see, e.g.,][]{georgiev2005sparse,comon2010handbook}, where we observe a series of multivariate signals $X_{i\cdot}\in\R^{p}$ for $i=1,2,...,n$, where $n$ is the number of observations.
In ICA, there exist $k$ independent, non-Gaussian and unobserved \textit{source} signals underlying each observation, $Z_{i\cdot} \in\R^k$ for $i=1,2,...,n$, and each observation is a linear mixture of these source signals, this is, $X=ZM\T$ (or $X_{i\cdot} = Z_{i\cdot}M$ for $i=1,2,...,n$), where $M\in\R^{p \times k}$ is the \textit{mixing} matrix. 
ICA aims to ``un-mix'' the observed $X$ and extract $Z$ from it. 
In particular, since the $k$ source signals are independent, it is often assumed that $Z$'s columns have unit length and are orthogonal to each other (i.e., $Z\in\V(n,k)$).
The ICA literature is rich in theoretical results \citep{hyvarinen2000independent, chen2006efficient, samworth2012independent, miettinen2015fourth}, and most methods for ICA (e.g. fastICA) identifies both platykurtic- and leptokurtic-sourced signals.

We consider a sparse version of ICA, sparse ICA, where $Z$ is sparse (or the columns of $Z$ follow leptokurtic distributions).
We show that sparse ICA and sparse PCA are unified by the SMA.
To see this, recall from Section \ref{sxn:sma} that the SMA of a data matrix is $ZBY\T$, where $Z$ and $Y$ are both sparse but $B$. 
We interpret the SMA for the two modern multivariate data analysis:
\begin{description}
	\item[Sparse PCA] For sparse PCA, we treat $Y$ as the sparse loadings, and $ZB$ together as the component scores. 
	\item[Sparse ICA] For sparse ICA, the sparse source signals (or the independent components) are the columns of $Z$, the mixing matrix is $BY\T$.
\end{description}
It can be seen that both sparse PCA and sparse ICA seek a sparse component in the data: sparse PCA extracts them for the column space ($Y$), while ICA the row space ($Z$).
Hence, performing sparse PCA to the transposed input data matrix actually accomplishes sparse ICA to the original data.
This highlights the similarities between sparse PCA and sparse ICA.

\subsection{Example: Blind source separation with SCA} \label{sxn:bss}
We apply SCA to the blind source separation of image data \citep{comon2010handbook}.
For example, suppose the source signals are individual images, and a sensor senses several mixed images, each an linear mixture of the sources. 
The objective is then to identify the source images from the observed ones (i.e., to decipher the linear coefficients).

\begin{figure}
	\centering
	\includegraphics[page=1,width=0.6\linewidth]{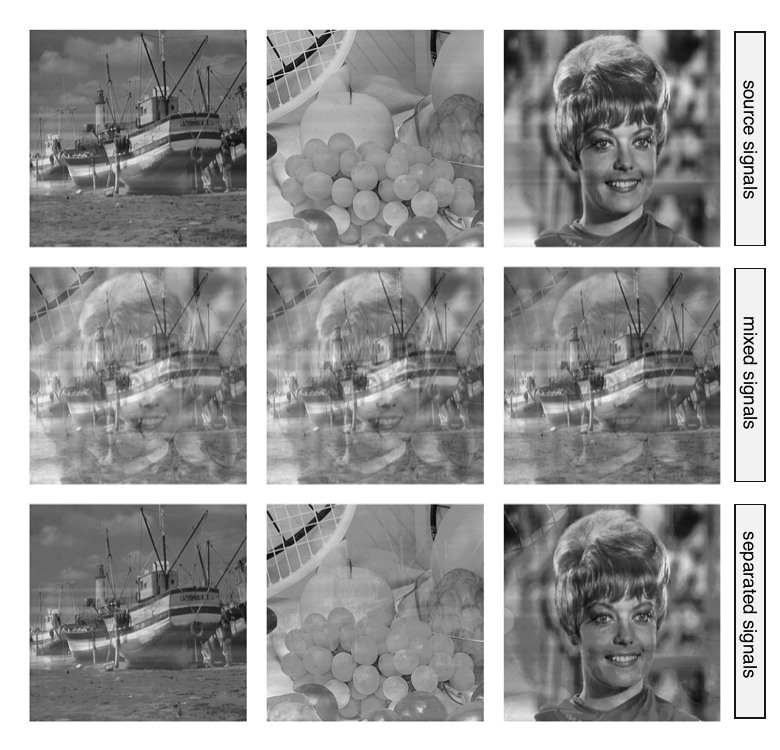}
	\caption{Blind image signal separation using SCA. The three panel rows display three source images, three linear mixtures of the source images, and the three separated images using SCA.}
	\label{fig:bss}
\end{figure}

We selected three $512 \times 512$-pixels pictures of diverse genres from the internet (Figure \ref{fig:bss}, the first row). 
The sample excess kurtosis of the images are 1.53, 3.32, and -0.45 respectively.
Next, we generated three ($n=3$) mixtures of the original images, with the linear coefficients randomly drawn from the uniform distribution, $\text{Unif(0,1)}$. The three mixed images are displayed in the second row of Figure \ref{fig:bss}. 
For sparse PCA, we vectorize the mixed images (that is $512^2$-pixels) and put them in a shallow matrix $X\in\R^{n \times p}$, where $p=262,144$.
This matrix is then input to SCA (Algorithm \ref{alg:sca}) for three sparse PCs ($k=3$), with the sparsity parameter $\gamma$ set to $\sqrt{nk}$. 
The resulting sparse loadings $Y\in\R^{p\times k}$ contains the three separated source images and the scores $S\in\R^{n\times k}$ decodes the mixing coefficients.
The third row in Figure \ref{fig:bss} displays the three separated images (i.e.,  the three rows of $Y$.)
The clean-cut identification of the source images suggests that sparse PCA is capable of extracting sparse and independent components from the data.

\subsection{Algorithmic comparisons}
Another insight for sparse PCA and sparse ICA can be gleaned from their algorithms. In this section, we demonstrate that the fastICA algorithm \citep{hyvarinen1999fast} and our SCA algorithm are both closely related to kurtosis \citep{mardia1970measures}.

The fastICA algorithm finds $Z$ in two steps. 
The first step is to pre-process $X$. The pre-processing of centering and whitening (see, e.g., \citet{comon1994independent}) results in the leading $k$ left singular vectors $\hat{U}\in\V(n,k)$.
The second steps searches for an orthogonal rotation that maximize the non-gaussianity of $\hat{U}R$, as measured by the approximation of negentropy, 
\begin{eqnarray} \label{eqn:ica1}
	\maxi_{R} &  \sum_{j=1}^k \left\{G({[\hat{U}R]}_{\cdot j})-G(\nu)\right\}^2 & \st R\in\U(k),
\end{eqnarray}
where $G(x)$ is a non-quadratic function for $x\in\R^n$, and $\nu\sim\text{N}(0,I_n)$ is the multivariate standard Gaussian vector.
Finally, $\hat{U}\hat{R}$ is the fastICA estimate for $Z$, where $\hat{R}$ is the solution to \eqref{eqn:ica1}.
\citet{hyvarinen1999fast} noted that setting $G(x) = \norm{x}_4^4/n$, the optimization in \eqref{eqn:ica1} takes the form\footnote{The authors also suggested different forms of $G(x)$.} 
\begin{eqnarray} \label{eqn:ica2}
	\maxi_{R} & \sum_{j=1}^k \text{kurt}^2 ({[UR]}_{\cdot j}) & \st R\in\U(k),
\end{eqnarray}
where $\text{kurt}(x)$ is the sample excess kurtosis of $x\in\R^n$ and is defined as $\text{kurt}(x)=n\sum_{i=1}^n (x_i-\bar{x})^4/\left(\sum_{i=1}^n (x_i-\bar{x})^2\right)^2-3$, where $\bar{x}=\sum_{i=1}^n x_i/n$ is the mean.
It can be seen from \eqref{eqn:ica2} that fastICA produces either leptokurtic ($\text{kurt}(x) > 0$) or platykurtic ($\text{kurt}(x) < 0$) estimation for the columns of $Z$, because of the squared kurtosis in the objective function. This primarily explains that fastICA allows both platykurtic- and leptokurtic-sourced signals.

As for SCA, the algorithm uses the varimax rotation to find the orthogonal rotation. Suppose $Y\in\V(n,k)$. Since the sum of squares of $Y$'s columns are constant, $\sum_{j=1}^k Y_{ij}^2=1$, maximizing the varimax rotation is equivalent to maximizing the sum of sample kurtosis of $Y$'s columns,
$$C_\text{varimax}(Y) = \sum_{j=1}^k \text{kurt}(Y_{\cdot j}) + \text{constant}.$$
This suggests that the varimax rotation in SCA promotes some leptokurtic columns in the loading $Y$ of sparse PCs.
Note that any sparse distribution is leptokurtic (see Theorem 2.1 of \citet{rohe2020vintage}). Hence, SCA generates specifically sparse PCs. 

In many applications of ICA, the number of independent components and the number of observed variables are the same (i.e., $p=k$), in which case, the mixing matrix is square. The $p=k$ regime is generally challenging. As such, many theoretical results presume no or very little noise in $X$, in order for estimating guarantees. 
By contrast, sparse PCA typically presumes the data to comprise noise and the statistical model usually contain a noise term. 
In addition, it is showed that sparse PCA is consistent even when the observed data is high-dimensional (i.e., $p$ grows at the same rate as $n$) or sparse by itself (i.e. contains many zeros) \citep{rohe2020vintage}, while it is unclear yet whether ICA is consistent or not under these settings.

\newpage
\setlength{\bibsep}{0pt plus 0.3ex}
\bibliographystyle{plainnat}
\bibliography{spca}

\end{document}